\pdfoutput=1
\documentclass{article}

\usepackage[utf8]{inputenc}
\usepackage[T1]{fontenc}
\usepackage{microtype}
\usepackage{graphicx}
\usepackage{tikz}
\usetikzlibrary{positioning,arrows.meta,decorations.pathreplacing,calligraphy}
\usepackage{wrapfig}

\usepackage{hyperref}
\hypersetup{pdfauthor={Mark Kozdoba, Shie Mannor},pdftitle={How Deep Are Deep GPs, Really? A Sharp Threshold and a Non-Gaussian Limit for Compositional GPs},pdfsubject={},pdfkeywords={},pdfcreator={},pdfproducer={}}

\usepackage[main,final]{neurips_2026}

\usepackage{amsmath}
\usepackage{amssymb}
\usepackage{amsthm}
\usepackage{mathtools}

\usepackage[capitalize,noabbrev]{cleveref}
\Crefname{appendix}{Supplementary}{Supplementaries}
\crefname{appendix}{supplementary}{supplementaries}

\newcommand{\RR}{ \mathbb{R} }

\newcommand{\Prob}[1]{\mathbb{P}\left( #1 \right)}

\newcommand{\Abs}[1]{\left| #1 \right|}
\newcommand{\Set}[1]{\left\{ #1 \right\}}
\newcommand{\Brack}[1]{\left( #1 \right)}

\newcommand{\Exp}[1]{ \mathbb{E} #1}
\newcommand{\Expsubidx}[2]{ \mathbb{E}_{#1} #2}

\newcommand{\norm}[1]{\left\|#1\right\|}

\newcommand{\eps}{\varepsilon}

\newlength{\dhatheight}

\newcommand{\cond}{\mid}

\newtheorem{theorem}{Theorem}[section]

\newtheorem{corollary}[theorem]{Corollary}

\newtheorem{remark}[theorem]{Remark}
\newtheorem{proposition}[theorem]{Proposition}

\newcommand{\mcX}{\mathcal{X}}

\newcommand{\mcZ}{\mathcal{Z}}
\newcommand{\mcN}{\mathcal{N}}
\newcommand{\mcR}{\mathcal{R}}
\newcommand{\mcU}{\mathcal{U}}
\newcommand{\mcP}{\mathcal{P}}

\newcommand{\bZ}{\bar{Z}}
\newcommand{\bz}{\bar{z}}
\newcommand{\bv}{\mathbf{v}}
\newcommand{\bu}{\mathbf{u}}
\newcommand{\bg}{\mathbf{g}}
\newcommand{\Cov}{\mathrm{Cov}}
\newcommand{\Var}{\mathrm{Var}}
\newcommand{\Law}{\mathrm{Law}}

\title{How Deep Are Deep GPs, Really? A Sharp Threshold and a Non-Gaussian Limit for Compositional GPs}

\author{%
  Mark Kozdoba \\
  Technion, IIT \\
  \texttt{markk@technion.ac.il} \\
  \And
  Shie Mannor \\
  Technion, IIT and NVIDIA \\
}

\begin{document}

\maketitle

 \begin{abstract}
    Compositional priors describe the generic properties of layered functions in deep Bayesian models, where deep neural networks with random weights are a canonical example.
    In the wide-network limit, the prior is a Gaussian process (GP) with a depth-dependent kernel, and its behaviour as depth grows has been extensively studied through this kernel.
    Here, we study another case, where each layer itself is a vector valued Gaussian process, and our aim is similarly to understand the limiting behaviour of the prior as depth grows.

    Previous GP work has established that for the radial basis function (RBF) kernel and a certain range of bandwidths $r$, the prior degenerates in the limit, converging to the set of constant functions --- which is not useful as a probabilistic model. In this paper we establish several new results.
    First, we identify a sharp bandwidth threshold $r_c(d) = \Theta(\sqrt{d})$ above which the limit is degenerate, strengthening the earlier bounds. Second, and more importantly, we show that for $r$ below the threshold $r_c(d)$ the prior converges to a limit distribution $\pi_{\bZ}$.  We also prove that these distributions are non-degenerate and non-Gaussian, with non-vanishing dependence between coordinates. In contrast to the previously known degenerate regime, deep Gaussian process priors can therefore admit non-trivial limits.

    Empirically, we verify the threshold across a range of dimensions $d$, and demonstrate a complex multimodal behaviour of the limit distributions $\pi_{\bZ}$ --- a regime that becomes increasingly narrow with $d$ and would be hard to identify without knowing the threshold.
\end{abstract}

\section{Introduction}
\label{sec:introduction}

A central approach to modelling uncertainty is Bayesian inference, in which epistemic uncertainty is carried by the posterior. The posterior, however, is shaped by the prior: the prior encodes the class of models one believes a priori plausible, and any uncertainty estimate extracted from the posterior inherits the prior's biases. Understanding the prior --- where its mass lives, what structure it prefers --- is thus a prerequisite to understanding Bayesian methods built on top of it. Some classical priors, such as shallow Gaussian processes (GPs) \citep{rasmussen2006gaussian} or composition of wide neural network layers \citep{neal1996bayesian,poole2016exponential,schoenholz2017deep,hayou2019impact}, are by now well understood. In particular, wide neural network priors are themselves a single Gaussian process, with a depth-dependent kernel \citep{lee2018deep,matthews2018gaussian} (see also \Cref{sec:previous_work}), 
and the properties of the prior can be read off from the properties of that kernel. However, less is known about compositions in which each layer is itself a Gaussian process.

\paragraph{Why compose Gaussian processes.}
Draw $G_1, \ldots, G_l$ i.i.d.\ from a centred vector valued Gaussian process prior and consider the composition $G_l \circ \cdots \circ G_1$. These \emph{compositional}, or \emph{deep}, Gaussian process priors \citep{damianou2013deep} have been shown to better fit empirical data than shallow GPs in some cases~\citep{damianou2013deep,lu2019interpretable}, with intermediate layers playing the role of learned features in a way that mirrors the argument for depth in neural networks. The price of this flexibility however is complexity of analysis: past a single layer, the joint law of the composed function at any finite set of inputs is no longer Gaussian, and its qualitative dependence on depth remains poorly understood.

\paragraph{The coupled walk and synchronisation.}
The central object in this paper is the Markov chain that the compositional prior induces on any collection of $n$ input points simultaneously. Let $\bZ_1 = (\bZ_{1,1},\ldots,\bZ_{1,n}) \in (\RR^d)^n$ be $n$ starting points and iterate
\begin{eqnarray}
  \label{eq:base_chain_equation}
  \bZ_{i+1} \;=\; G_i \bZ_i \;=\; (G_i \bZ_{i,1}, \ldots, G_i \bZ_{i,n}),
\end{eqnarray}
with $G_i$ i.i.d.\ samples from a centred vector GP with rotation-invariant kernel $k(x,x') = \exp(-\norm{x-x'}^2 / (2 r^2))$. That is, $G_i:\RR^d \rightarrow \RR^d$ are random Gaussian functions (see Section \ref{sec:results} for a definition); see \Cref{fig:composition} for the structure.
Because the \emph{same} $G_i$ is applied to every coordinate, the $n$ trajectories $\Set{\bZ_{i,t}}_{i \ge 1}$, where $t\leq n$, are coupled through the shared randomness. \citet{dunlop2018how} showed that for $r$ large enough the process has a \emph{degenerate} limit: the pairwise distances $\norm{\bZ_{i,s} - \bZ_{i,t}}$ collapse to zero with depth (i.e.\ $i \rightarrow \infty$), and consequently, the prior concentrates on constant functions. Note that when the bandwidth $r$ is large, the values of the process $G_i$ are strongly correlated, and therefore the process itself is somewhat close to the constant functions. The above result then shows that composition further amplifies this, yielding a limit that coincides with the constant functions. 
 However, their argument leaves open both whether their particular  threshold for $r$ is sharp, and, more importantly, what happens on the other side of it.

 We note that the general framework of coupled walks \eqref{eq:base_chain_equation} is
 also studied in the Theory of Iterated Random Functions (see \Cref{sec:previous_work}); there, the phenomenon of a degenerate limit is known as \emph{synchronisation}, with the interpretation that all particles eventually follow the same trajectory. We will use the terms synchronisation and degenerate limit interchangeably.

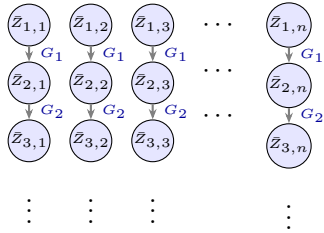
\begin{wrapfigure}{l}{0.5\textwidth}
  \centering
  \begin{tikzpicture}[
      node distance = 2mm and 2.5mm,
      pt/.style = {circle, draw, fill=blue!10, inner sep=0pt, minimum size=4.2mm,
                   font=\tiny},
      ar/.style = {-{Stealth[length=1.4mm]}, thick, gray},
      lab/.style = {font=\tiny\itshape}
    ]
    \node[pt] (z11) {$\bZ_{1,1}$};
    \node[pt, right=of z11] (z12) {$\bZ_{1,2}$};
    \node[pt, right=of z12] (z13) {$\bZ_{1,3}$};
    \node[right=of z13] (zd1) {$\cdots$};
    \node[pt, right=of zd1] (z1n) {$\bZ_{1,n}$};
    \node[pt, below=of z11] (z21) {$\bZ_{2,1}$};
    \node[pt, below=of z12] (z22) {$\bZ_{2,2}$};
    \node[pt, below=of z13] (z23) {$\bZ_{2,3}$};
    \node[below=of zd1] (zd2) {$\cdots$};
    \node[pt, below=of z1n] (z2n) {$\bZ_{2,n}$};
    \node[pt, below=of z21] (z31) {$\bZ_{3,1}$};
    \node[pt, below=of z22] (z32) {$\bZ_{3,2}$};
    \node[pt, below=of z23] (z33) {$\bZ_{3,3}$};
    \node[below=of zd2] (zd3) {$\cdots$};
    \node[pt, below=of z2n] (z3n) {$\bZ_{3,n}$};
    \foreach \a/\b in {z11/z21, z12/z22, z13/z23, z1n/z2n}{
      \draw[ar] (\a) -- node[midway, right=0.5pt, lab, blue!60!black] {$G_1$} (\b);
    }
    \foreach \a/\b in {z21/z31, z22/z32, z23/z33, z2n/z3n}{
      \draw[ar] (\a) -- node[midway, right=0.5pt, lab, blue!60!black] {$G_2$} (\b);
    }
    \node[below=1.5mm of z31] {$\vdots$};
    \node[below=1.5mm of z32] {$\vdots$};
    \node[below=1.5mm of z33] {$\vdots$};
    \node[below=1.5mm of z3n] {$\vdots$};
  \end{tikzpicture}
  \caption{The compositional chain~\eqref{eq:base_chain_equation} on $n$ input points. At depth $i$ the \emph{same} random Gaussian function $G_i$ is applied to every one of the $n$ points, coupling all trajectories through the shared randomness.}
  \label{fig:composition}
\end{wrapfigure}

\paragraph{Convergence and sharp thresholds.}
We give a sharp almost-sure threshold separating a supercritical regime, in which the chain synchronises, from a subcritical regime, in which it admits a nontrivial stationary law. The critical radius is
\begin{eqnarray*}
  r_c(d) \;=\; \sqrt{2}\,e^{\psi(d/2)/2},
\end{eqnarray*}
where $\psi$ is the digamma function; asymptotically $r_c(d) \sim \sqrt{d}$ as $d \to \infty$ (see \Cref{rem:sync_d_scaling}). For $r > r_c(d)$ the coupled walk synchronises almost surely at an explicit exponential rate (\Cref{thm:sync_v_convergence_d}), improving on \citet{dunlop2018how}, whose threshold lies strictly above $r_c(d)$ for every $d$. More importantly, for $r < r_c(d)$ --- a regime not addressed by \citet{dunlop2018how} --- we show that the scalar pairwise-distance chain converges in total variation (TV) to a unique nontrivial stationary law on $(0,\infty)$ (\Cref{thm:sync_v_convergence_d}), and the full position chain $\bZ_i$ converges in total variation to a unique stationary law $\pi_{\bZ}$ on $(\RR^d)^n$ for every $n \ge 2$ (\Cref{thm:sync_Z_convergence}).

We note that \citet{dunlop2018how} also proved convergence to stationarity, but for a different class of processes. Their direct approach, of constructing Lyapunov functions directly for $\bZ$, does not seem to apply to the composition class, for which they only prove synchronization as discussed above. We avoid the issue by taking an indirect approach: observing that $\bZ$ is determined by the pairwise distances $\norm{\bZ_{i,t} - \bZ_{i,s}}$, we work with the induced chain on those distances. Methodologically, \Cref{thm:sync_Z_convergence} is an extension of \Cref{thm:sync_v_convergence_d}.

\paragraph{The limit is non-Gaussian with non-trivial dependence.}
A priori one might expect the stationary distribution $\pi_{\bZ}$ to be Gaussian --- it is built entirely from Gaussian ingredients, and each marginal $\bZ_{\infty,t}$ is standard Gaussian. Remarkably, however, the joint law of any two coordinates $(\bZ_{\infty,s}, \bZ_{\infty,t})$ is not jointly Gaussian (\Cref{prop:sync_not_gaussian}).
Moreover, we show that there is dependence between coordinates of $\pi_{\bZ}$, and that it can be estimated in terms of $r$ and $d$ (\Cref{rem:sync_indep_gaussian_comparison}). Setting $r = \lambda\, r_c(d)$, the dependence is weak at small $\lambda$ and grows strong as $\lambda$ approaches $1$. Crucially, the range of $\lambda$ in which it is visible on a finite sample of $\pi_{\bZ}$ shrinks rapidly with $d$ --- at $d = 100$ the entire non-trivial regime sits inside roughly $\lambda \in (0.99, 1)$ (\Cref{sec:experiments}). Without the bounds of \Cref{cor:sync_log_stationary} and \Cref{rem:sync_indep_gaussian_comparison} and the precise value of $r_c(d)$, this narrow strip is hard to locate --- a moderate-$\lambda$ experiment in high $d$ looks indistinguishable from the independent-coordinate Gaussian baseline. However, when the relevant scale of $1-\lambda$ is identified, $\lambda$ can be used to tune the strength of dependence. 

\paragraph{$\pi_{\bZ}$ as a generative prior.}
One use of the limit laws, beyond the analysis of depth itself, is as priors for \emph{generative} Bayesian models, as opposed to supervised Bayesian learning \citep{damianou2013deep,damianou2016variational}. The laws $\pi_{\bZ}$ are exchangeable: the chain \eqref{eq:base_chain_equation} treats the $n$ coordinates symmetrically, so its unique stationary law is invariant under permutations of the sample indices, and the family is consistent across $n$, since the marginal of $\pi_{\bZ}$ on any $k$ coordinates is stationary for the $k$-point chain and hence equals its unique stationary law. By de Finetti's theorem, $\pi_{\bZ}$ is therefore a mixture of i.i.d.\ laws,
\begin{eqnarray*}
  \pi_{\bZ} \;=\; \int_{\mcP(\RR^d)} M^{\otimes n}\, dQ(M),
\end{eqnarray*}
that is, to sample from $\pi_{\bZ}$, draw a distribution $M$ on $\RR^d$ from a prior $Q$ and then $n$ i.i.d.\ points from $M$. \Cref{fig:exp_pi_Z_tsne_main} effectively shows the shape of several such $M$, each chain being one realisation, and in particular that non-trivial $M$ exist. Given $n$ data points one would then infer $M$ and generate new samples from it, as is done at finite depth, with an approximate posterior sampling scheme, by \citet{damianou2013deep}. Two features make the infinite-depth prior attractive for this. A limit exists, so quantities cannot improve or deteriorate indefinitely with depth, which is a first step toward understanding why a particular finite depth should be used. And the limit is invariant under one layer, so a step back through the layers, which posterior sampling requires, maps $\pi_{\bZ}$ to itself rather than to a different law at every layer.

\paragraph{Experiments}
We numerically verify the threshold $r_c(d)$ across $d \in \Set{1, 10, 100}$, and sample from $\pi_{\bZ}$ at $\lambda$ values selected per dimension using the bounds discussed above. The samples display complex chain-specific multimodal structure, consistent with the non-Gaussianity of \Cref{prop:sync_not_gaussian}.

\paragraph{Summary of contributions.}
Our main contributions are: (i) convergence in total variation to a unique stationary law $\pi_{\bZ}$ at every $r < r_c(d)$. Until now, the only depth-asymptotic behaviour established for the composition class was synchronisation onto constants; we show that below the threshold the prior in fact admits a non-degenerate limit, qualitatively changing the picture; (ii) the structural finding that $\pi_{\bZ}$ is non-Gaussian, with coordinate dependence that does not wash out as $d \to \infty$; and (iii) a sharp almost-sure threshold $r_c(d)$ separating the synchronised regime from this non-trivial limit, with an explicit decay rate above the threshold. We verify (iii) across $d \in \Set{1, 10, 100}$, and exhibit (i)--(ii) by visualising samples from $\pi_{\bZ}$ at $\lambda$ values selected per dimension using the bounds of (iii); the samples display complex chain-specific multimodal structure.

\paragraph{Organization.}
The rest of this paper is organized as follows. \Cref{sec:previous_work} reviews related work on deep and compositional Gaussian processes, infinite-width neural-network priors, iterated random functions, and previous work on GP composition. \Cref{sec:definitions} fixes notation and recalls the objects we use. \Cref{sec:results} contains the formal statements of our main results together with proof sketches. \Cref{sec:experiments} presents numerical experiments verifying the threshold across dimensions and \Cref{sec:conclusions} discusses open directions. The  proofs are collected in \Cref{sec:proofs}.

\section{Previous Work}
\label{sec:previous_work}

\paragraph{Deep and compositional Gaussian processes.}
Deep Gaussian processes were introduced by \citet{damianou2013deep} as a hierarchical prior formed by stacking GP layers. Their expressivity --- non-stationary effective kernels, input-dependent length scales, richer induced push-forward distributions --- has been the subject of a substantial body of follow-up work, including approximate inference methods \citep{bui2016deep,salimbeni2017doubly,cutajar2017random,havasi2018inference}, conditional moment calculations \citep{lu2019interpretable}, and depth-asymptotic studies of several specific constructions \citep{duvenaud2014avoiding,dunlop2018how}. Among the constructions in this literature, the \emph{composition} class --- in which a single GP sample is applied layer-wise to the coupled input vector --- is the object of our analysis.

\paragraph{Infinite-width neural-network priors.}
A parallel mean-field literature analyses depth asymptotics of Bayesian neural-network priors with i.i.d.\ Gaussian weights and biases of variances $\sigma_w^2$ and $\sigma_b^2$. \citet{poole2016exponential, schoenholz2017deep, hayou2019impact} identify, at infinite width and for tanh activations, a synchronisation / non-synchronisation dichotomy in $(\sigma_w^2, \sigma_b^2)$ analogous to ours: in some parameter regimes the network synchronises, i.e.\ the covariance of two input points approaches $1$, while in others the covariance approaches a strictly smaller fixed point. For ReLU activations the dichotomy collapses and synchronisation is unavoidable for every choice of these parameters \citep{hayou2019impact}.

A structural difference from our work is that in this setting the entire infinite-width composition collapses to a single Gaussian process, with all depth dependence absorbed into its kernel, termed the \emph{neural network Gaussian process} (NNGP) kernel, which satisfies a deterministic depth recursion \citep{matthews2018gaussian, lee2018deep, yang2019scaling}. The synchronisation behaviour is therefore a property of the depth-evolution of this kernel. Thus, the setup here can be considered as a single Gaussian process whose kernel encodes the entire depth dependence. In contrast, the composition we analyse is a genuine composition of many GPs, and is non-Gaussian at every finite depth. In that setting \Cref{thm:sync_v_convergence_d} establishes the existence of a non-trivial limiting law, and we further show that this law is also non-Gaussian.

\paragraph{Non-Gaussian limits of deep priors.}
Two recent lines of work in the deep-neural-network literature establish non-Gaussian limits in different parameter regimes. \citet{bordino2023infinitely}, building on \citet{peluchetti2020stable}, replace Gaussian weights with heavy-tailed Stable weights and obtain Stable-process limits in the infinite-width regime. \citet{hayou2022infinitedepth} fix the width and let depth grow in a residual architecture, obtaining continuous-time SDE limits whose form depends on the activation function \citep[see also][]{li2022neuralcovariance}. Both routes deliver non-Gaussianity by changing the mechanism behind the standard Gaussian-limit picture: \citet{bordino2023infinitely} change the noise mechanism, replacing Gaussian weights with Stable ones to begin with; \citet{hayou2022infinitedepth} change the iteration mechanism, since the residual block is additive rather than compositional and is rescaled with the total depth. These architectures are different from our approach, which is purely compositional and uses only Gaussian ingredients. The non-Gaussianity in \eqref{eq:base_chain_equation} arises directly from the GP composition mechanism.

\paragraph{Iterated random functions.}
Another area in which the chain \eqref{eq:base_chain_equation} appears is the theory of iterated random functions \citep{diaconis1999iterated,wushao2004limit,stenflo2012survey}, which studies Markov chains obtained by iterating maps $G_i$ drawn i.i.d.\ from a distribution on function space. The classical theory focuses on contraction-on-average criteria and backward-iteration limits, primarily for finite-dimensional parametrised families such as random affine maps or Kesten-type perpetuities. The specific setting we study --- $G_i$ a sample from a Gaussian process with spatially stationary kernel --- gives a genuinely infinite-dimensional random function with non-trivial spatial correlations, and the sharp-threshold / non-Gaussian-limit phenomena established here are not addressed in the classical literature.

\paragraph{On the composition class.}
The closest technical precedent to our work is \citet{dunlop2018how}, working with the scalar pairwise-distance chain (analogous to $v$ in \Cref{thm:sync_v_convergence_d}). They prove synchronization in $L^2$ at threshold $r > \sqrt{d}$, which upgrades to almost-sure convergence but only at this cruder $L^2$-driven threshold. \citet{tong2021characterizing} refine this moment argument into a nonlinear recurrence for $\Exp{v_i^2}$ and extend it to further stationary kernels. For the RBF kernel their threshold is again $r = \sqrt{d}$, and a recurrence for a single moment does not address the subcritical regime. \Cref{thm:sync_v_convergence_d} sharpens the a.s.\ threshold to $r_c(d) = \sqrt{2}\, e^{\psi(d/2)/2}$, opening the window $r \in (r_c(d), \sqrt{d})$ of almost-sure synchronization that is invisible to any $L^2$ argument. As noted in the introduction, $r_c(d) = \Theta(\sqrt{d})$, with $r_c(d) < \sqrt{d}$ strictly for every $d \ge 1$ but $r_c(d)/\sqrt{d} \to 1$ as $d \to \infty$ (\Cref{rem:sync_d_scaling}); the gap with Dunlop's threshold therefore closes with dimension. More importantly, \citet{dunlop2018how} do not address the subcritical regime $r < r_c(d)$, which is covered by \Cref{thm:sync_v_convergence_d} and, along with \Cref{thm:sync_Z_convergence,prop:sync_not_gaussian}, is our principal contribution beyond their work.

As discussed earlier, \citet{dunlop2018how} show ergodicity results (their Theorem~8), but for a \emph{different} construction (Paciorek-style hierarchical kernels), and these do not transfer to the composition class.

\citet{lu2019interpretable} compute, in closed form, the second and fourth moments of a single two-layer composition for several specific kernel families. These are one-step identities and cannot be iterated to larger depth: the moments at layer $L+1$ depend on the full distribution of the layer-$L$ output, not on its moments alone, since the chain is non-Gaussian at every finite depth and no finite collection of moments closes the recursion.

\section{Notation and preliminaries}
\label{sec:definitions}

In this section we collect a few standard objects used throughout. The \emph{digamma function} $\psi$ is the logarithmic derivative of the gamma function, $\psi(z) := \Gamma'(z) / \Gamma(z)$ (\cite{abramowitz1964handbook}). The \emph{Euler--Mascheroni constant} is $\gamma := -\psi(1) \approx 0.5772$, and we will use the value $\psi(1/2) = -\gamma - 2\log 2$.

For $d \ge 1$, the \emph{chi-squared distribution with $d$ degrees of freedom}, denoted $\chi^2_d$, is the law of $\sum_{k=1}^d g_k^2$ for $g_1, \ldots, g_d \sim \mcN(0, 1)$ i.i.d.\ standard Gaussians (\cite{abramowitz1964handbook}). Its Lebesgue density on $(0, \infty)$ is
 $ f_d(x) \;=\; \frac{1}{2^{d/2}\,\Gamma(d/2)}\, x^{d/2 - 1}\, e^{-x/2},$
and the logarithmic moment is $\Exp{\log X} = \psi(d/2) + \log 2$ for $X \sim \chi^2_d$. In particular $\chi^2_1$ is the law of $g^2$ for a single $g \sim \mcN(0,1)$, and $\Exp{\log g^2} = \psi(1/2) + \log 2 = -\gamma - \log 2$.

For matrices $A \in \RR^{m \times m}$ and $B \in \RR^{p \times p}$, the \emph{Kronecker product} $A \otimes B \in \RR^{mp \times mp}$ is the block matrix $(A \otimes B)_{(i,k),(j,l)} = A_{ij} B_{kl}$ for $1 \le i, j \le m$ and $1 \le k, l \le p$. We use it to describe joint Gaussian laws on $(\RR^d)^n$: if $R \in \RR^{n \times n}$ is positive semidefinite and we identify $\RR^{nd}$ with $(\RR^d)^n$ by stacking, then $\mcN(0,\, R \otimes I_d)$ is the law of a random vector $\bz = (z_1, \ldots, z_n) \in (\RR^d)^n$ whose blocks are jointly centred Gaussian with cross-covariance $\Cov(z_s, z_t) = R_{st}\, I_d$; equivalently, the $d$ output coordinates are i.i.d.\ across $k = 1, \ldots, d$, each an $\RR^n$-valued centred Gaussian with covariance $R$.

\section{Results}
\label{sec:results}

We consider the chain \eqref{eq:base_chain_equation} driven by Gaussian process innovations. Specifically, each $G_i$ is sampled i.i.d.\ from a centred vector-valued Gaussian process on $\RR^d$ whose $d$ output coordinates $G_i^{(1)}, \ldots, G_i^{(d)}$ are themselves i.i.d.\ scalar centred GPs with the rotation-invariant radial basis function (RBF) kernel
\begin{eqnarray*}
  k(x, x') \;=\; \phi(\norm{x - x'}), \qquad \phi(t) \;:=\; e^{-t^2/(2r^2)}.
\end{eqnarray*}
This setting is similar to the one considered in \cite{dunlop2018how}, although they consider a slightly more general family of kernels $k$.
At any fixed $x \in \RR^d$ the evaluation has law $G_i(x) \sim \mcN(0, \phi(0)\, I_d) = \mcN(0, I_d)$, independent of $\bZ_i$, so every marginal trajectory satisfies $\bZ_{i,t} \sim \mcN(0, I_d)$ at every depth $i \ge 1$. The depth-dependence lies entirely in the \emph{joint} law across the $n$ trajectories.

The first object we study is the pairwise distance for a fixed pair of coordinates $s \ne t$, held fixed throughout this and the next subsection. Let $V_i := \bZ_{i,s} - \bZ_{i,t} \in \RR^d$ and $v_i := \norm{V_i}$. Conditional on $\bZ_i$, the pair $(\bZ_{i+1,s}, \bZ_{i+1,t})$ is jointly centred Gaussian with cross-covariance $\phi(v_i)\, I_d$, so
\begin{eqnarray}
  \label{eq:kernel_corr_converegnce_d}
  v_{i+1}^2 \;=\; 2\Brack{1 - \phi(v_i)}\, X_i, \qquad X_i \sim \chi^2_d,
\end{eqnarray}
with $X_i$ independent of $v_i$ by rotation invariance. Because the $G_i$ are i.i.d.\ across $i$, the innovations $\Set{X_i}_{i \ge 1}$ at this fixed pair are i.i.d., and the scalar chain $\Set{v_i^2}$ is Markov on $[0, \infty)$.

The behaviour of \eqref{eq:kernel_corr_converegnce_d} is governed by the local contraction rate near $0$: for $v_i$ small, $1 - \phi(v_i) \approx v_i^2 / (2r^2)$, so
\begin{eqnarray}
  \label{eq:near_origin_lin}
  v_{i+1}^2 \;\approx\; \frac{v_i^2}{r^2}\, X_i \qquad \text{for small } v_i,
\end{eqnarray}
and on the log-chain $L_i := \log v_i^2$ the increment $L_{i+1} - L_i \approx \log X_i - 2\log r$ is i.i.d.\ with mean
\begin{eqnarray}
  \label{eq:rho_d_def}
  \rho_d \;:=\; \psi(d/2) + \log 2 - 2\log r
\end{eqnarray}
(using $\Exp{\log X} = \psi(d/2) + \log 2$; see \Cref{sec:definitions}). The sign of $\rho_d$ separates two regimes, made precise in the following theorem.

\begin{theorem}[Sharp dichotomy for the scalar pairwise-distance chain]
  \label{thm:sync_v_convergence_d}
  Let
  \begin{eqnarray*}
    r_c(d) \;:=\; \sqrt{2}\, e^{\psi(d/2)/2}
  \end{eqnarray*}
  (where $\psi$ is the digamma function), and let $\Set{v_i^2}$ be the Markov chain on $[0, \infty)$ given by \eqref{eq:kernel_corr_converegnce_d}.
  \begin{itemize}
    \item[(i)] If $r > r_c(d)$, then for every $v_1 \in \RR$, $v_i \to 0$ almost surely, exponentially fast: with $\rho_d$ as in \eqref{eq:rho_d_def},
      \begin{eqnarray*}
        \limsup_{i \to \infty} \frac{1}{i}\, \log v_i^2 \;\le\; \rho_d \;=\; \psi(d/2) + \log 2 - 2\log r \;<\; 0 \qquad \text{a.s.}
      \end{eqnarray*}
    \item[(ii)] If $r < r_c(d)$ and $v_1 \ne 0$, then $v_i \not\to 0$ almost surely, $\Set{v_i^2}$ admits a unique nontrivial stationary distribution $\pi^{\bv}$ on $(0, \infty)$, and the law of $v_i^2$ converges to $\pi^{\bv}$ in total variation.
  \end{itemize}
\end{theorem}

\begin{remark}
  \label{rem:sync_initial_condition}
  The asymptotic behaviour is independent of the initial condition $v_1$: in (i) the exponential decay holds for every $v_1 \in \RR$ with a $v_1$-independent rate bound.
\end{remark}

\paragraph{Sketch: supercritical regime ($\rho_d < 0$).}
The pointwise inequality $1 - e^{-x} \le x$ upgrades the near-origin linearization \eqref{eq:near_origin_lin} to a \emph{global} bound $v_{i+1}^2 \le v_i^2\, X_i / r^2$, valid for every $v_i$. Iterating and taking logs gives $L_i \le L_1 + \sum_{j<i}(\log X_j - 2\log r)$; the Kolmogorov strong law of large numbers (SLLN) on the i.i.d.\ sequence $\Set{\log X_j - 2\log r}$ then forces $\limsup_i L_i/i \le \rho_d < 0$ a.s., so $v_i \to 0$ exponentially fast at rate at least $|\rho_d|$.

\paragraph{Sketch: subcritical regime ($\rho_d > 0$).}
Near $L_i = -\infty$ the linearization gives positive mean increment $\rho_d$, so the SLLN forces the log-chain to leave any far-left region in finite time; thus $v_i \not\to 0$. At the other end, the global bound $v_{i+1}^2 \le 2 X_i$ (from $1 - e^{-x} \le 1$) caps the chain by an i.i.d.\ envelope independent of $v_i$, providing downward drift from far right. Using these facts we construct a Foster--Lyapunov function decreasing at both ends of $L$; together with $\psi$-irreducibility from the strictly positive transition density on $(0, \infty)$, standard arguments then yield positive Harris recurrence and a unique invariant probability on $(0, \infty)$ (see \cite{meyn1993markov}).

\paragraph{Convergence of $\bZ_i$ for general $n$.}
\Cref{thm:sync_v_convergence_d} describes the scalar dynamics at any single pair $(s, t)$. Controlling the full position chain $\bZ_i \in (\RR^d)^n$ requires more: each pairwise distance $u_{st}^{(i)} = \norm{\bZ_{i,s} - \bZ_{i,t}}^2$ follows the scalar kernel marginally, but the pairs are coupled through the shared GP innovation $G_i$, so a unique joint stationary law needs a Foster--Lyapunov argument controlling all pairs simultaneously. We carry this out on the position chain itself, via a sum-of-squared-log-distances Lyapunov function whose drift decomposes pair-by-pair, while the Gaussian transition kernel supplies irreducibility and a small set on $(\RR^d)^n$. This yields a unique stationary law $\pi_{\bZ}$ and total-variation convergence of $\bZ_i$, with the distance vector $\bu^{(i)}$ inheriting a unique stationary law $\pi^{\bu}$ (all pairwise marginals equal to $\pi^{\bv}$) as its image. We work on $\bZ_i$ rather than directly on $\bu^{(i)}$ because for $n \ge d+2$ the distance vector is confined to a lower-dimensional set and carries no Lebesgue density (\Cref{rem:edm_state_space}).

\begin{theorem}[TV convergence of $\bZ_i$, general $n$]
  \label{thm:sync_Z_convergence}
  Let $\mcZ = \RR^d$, $n \ge 2$, and $G_i$ as in \Cref{thm:sync_v_convergence_d}; assume $r < r_c(d)$. The pairwise-distance chain
  \begin{eqnarray*}
    \bu^{(i)} \;:=\; \Brack{u_{st}^{(i)}}_{s<t}, \qquad u_{st}^{(i)} := \norm{\bZ_{i,s} - \bZ_{i,t}}^2,
  \end{eqnarray*}
  is Markov on the set $\mcU_{n,d} \subset (0,\infty)^{\binom{n}{2}}$ of squared-distance vectors of $n$ pairwise-distinct points in $\RR^d$, and admits a unique stationary distribution $\pi^{\bu}$ on $\mcU_{n,d}$ whose pairwise coordinate marginals all equal $\pi^{\bv}$ (the scalar stationary law from \Cref{thm:sync_v_convergence_d}~(ii)). The full position chain $\Set{\bZ_i}$ on $(\RR^d)^n$ admits a unique stationary law
  \begin{eqnarray}
    \label{eq:sync_Z_limit_measure}
    \pi_{\bZ} \;:=\; \int_{\mcU_{n,d}} \mcN\!\bigl(0,\, R(\bu) \otimes I_d\bigr)\, d\pi^{\bu}(\bu),
  \end{eqnarray}
  with $R(\bu)_{st} = \phi(\sqrt{u_{st}})$, and for every initial condition $\bZ_1 \in (\RR^d)^n$ with $\bZ_{1,s} \ne \bZ_{1,t}$ for all $s \ne t$,
  \begin{eqnarray}
    \label{eq:sync_Z_tv_bound}
    \norm{\Law(\bZ_i) - \pi_{\bZ}}_{\mathrm{TV}} \;\le\; \norm{\Law(\bu^{(i-1)}) - \pi^{\bu}}_{\mathrm{TV}} \;\xrightarrow{i \to \infty}\; 0.
  \end{eqnarray}
\end{theorem}

\begin{remark}[State space of the distance chain]
  \label{rem:edm_state_space}
  For $n \le d+1$ the set $\mcU_{n,d}$ is a full-dimensional subset of $(0,\infty)^{\binom{n}{2}}$, but for $n \ge d+2$ it is constrained: squared Euclidean distance matrices of $n$ points in $\RR^d$ obey rank (Cayley--Menger) conditions and fill out only a submanifold of dimension $nd - \binom{d+1}{2} < \binom{n}{2}$. The law of $\bu^{(i)}$ is then singular with respect to Lebesgue measure on $(0,\infty)^{\binom{n}{2}}$, so $\Set{\bu^{(i)}}$ is not irreducible there. We therefore establish ergodicity on the position chain $\Set{\bZ_i}$ on $(\RR^d)^n$, whose Gaussian kernel $\mcN(0, R(\bu(\bz)) \otimes I_d)$ has a strictly positive density, and obtain $\pi^{\bu}$ as its image under $\bz \mapsto \bu(\bz)$; see \Cref{sec:proofs}. The experiments all operate in this regime ($n$ in the hundreds or more, $d \le 100$).
\end{remark}

\begin{figure}[t]
  \centering
  \includegraphics[width=\textwidth]{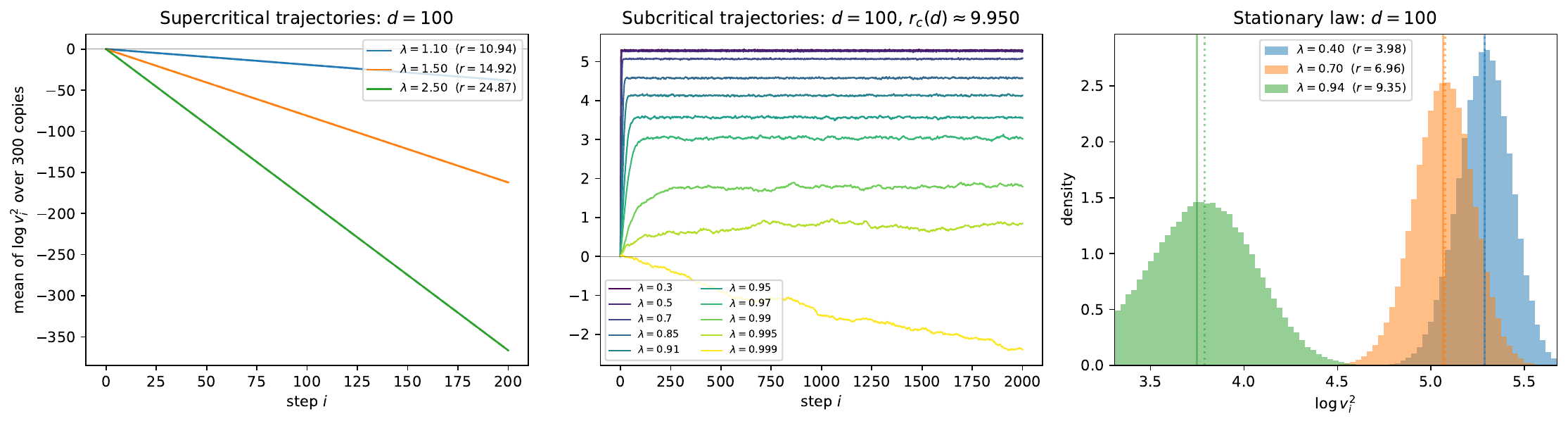}
  \caption{The dichotomy of \Cref{thm:sync_v_convergence_d} at $d = 100$. \emph{Left (supercritical):} trajectories of mean $\log v_i^2$ over $300$ i.i.d.\ copies starting at $v_1 = 1$, for $\lambda \in \Set{1.1, 1.5, 2.5}$; the grey dashed reference line has slope $-2\log\lambda$ predicted by \Cref{thm:sync_v_convergence_d}~(i) and is hidden behind the simulated curve in every case. \emph{Middle (subcritical):} the same kind of trajectories for ten $\lambda$ values spanning $0.30$ to $0.999$ (purple to yellow). All but the near-critical $\lambda = 0.999$ have equilibrated within the first few hundred steps; $\lambda = 0.999$ continues to drift, consistent with the slowdown of \Cref{thm:sync_v_convergence_d}~(ii) as $\lambda \uparrow 1$. \emph{Right:} stationary law of $\log v_i^2$ at $\lambda \in \Set{0.40, 0.70, 0.94}$, from a chain of $2 \cdot 10^5$ samples after $2 \cdot 10^4$-step burn-in. Solid verticals: empirical stationary mean. Dotted verticals: Jensen upper bound $L_*(r, d)$ from \Cref{cor:sync_log_stationary}~(b). The bound is nearly tight at every $\lambda$. The $d = 1, 10$ companions are in \Cref{sec:appendix_more_figs}.}
  \label{fig:exp_d100_main}
\end{figure}

\begin{theorem}[$\pi_{\bZ}$ is marginally Gaussian but not jointly Gaussian]
  \label{prop:sync_not_gaussian}
  Assume $r < r_c(d)$, and let $(\bZ_{\infty,1}, \ldots, \bZ_{\infty,n}) \sim \pi_{\bZ}$ be a sample from the stationary law of \Cref{thm:sync_Z_convergence}. Each coordinate is marginally standard Gaussian,
  \begin{eqnarray*}
    \bZ_{\infty,t} \;\sim\; \mcN(0, I_d) \qquad \text{for every } t \in \Set{1, \ldots, n},
  \end{eqnarray*}
  but for every pair $s \ne t$ the joint distribution of $(\bZ_{\infty,s}, \bZ_{\infty,t})$ is \emph{not} Gaussian. Consequently $\pi_{\bZ}$ itself is not jointly Gaussian: if it were, every pair would be too, contradicting the previous claim.
\end{theorem}

The pairwise non-Gaussianity holds because, as we show, no distribution of $v_i$ arising from a Gaussian difference of two coordinates can be stationary for \eqref{eq:kernel_corr_converegnce_d}. The argument is an $n = 2$ argument: any fixed pair $(s, t)$ evolves as the two-point chain, and the general-$n$ statement of \Cref{thm:sync_Z_convergence} is used only to guarantee that the limit law $\pi_{\bZ}$ under which the pair is examined exists.

\paragraph{Argument for the proof of \Cref{prop:sync_not_gaussian}.}
\emph{(1, marginal.)} As discussed above, the marginal $\bZ_{i,t}$ is Gaussian, $\bZ_{i,t} \sim \mcN(0, I_d)$, at every finite step, and the TV convergence $\Law(\bZ_i) \to \pi_{\bZ}$ from \Cref{thm:sync_Z_convergence} forces every coordinate marginal under $\pi_{\bZ}$ to be standard Gaussian.
\emph{(2, isotropy.)} Fix any $s \ne t$. The one-step dynamics make $V_i := \bZ_{i,s} - \bZ_{i,t}$ rotationally invariant in $\RR^d$ at every finite $i$ (its conditional covariance given $\bZ_{i-1}$ is a scalar multiple of $I_d$), and rotational invariance passes to the weak limit $V_\infty$. If now $(\bZ_{\infty,s}, \bZ_{\infty,t})$ were Gaussian, $V_\infty$ would be a centred isotropic Gaussian, hence $V_\infty \sim \mcN(0, \sigma_v^2\, I_d)$ for some $\sigma_v^2 \in (0, \infty)$, and $v_\infty^2 := \norm{V_\infty}^2$ would be a scaled $\chi^2_d$.
\emph{(3, punchline.)} No scaled $\chi^2_d$ is stationary for \eqref{eq:kernel_corr_converegnce_d} --- proven via characteristic functions in \Cref{sec:proofs} --- contradicting the stationarity of $\pi^{\bv}$ in \Cref{thm:sync_v_convergence_d}~(ii). The pair $(s, t)$ was arbitrary, so this rules out joint Gaussianity for every pair under $\pi_{\bZ}$, and hence for $\pi_{\bZ}$ itself.

\begin{remark}
  \label{rem:sync_d_scaling}
  The map $d \mapsto r_c(d) = \sqrt 2\, e^{\psi(d/2)/2}$ is increasing, and asymptotically $\psi(d/2) = \log(d/2) - 1/d + O(1/d^2)$, so $r_c(d) \sim \sqrt d$ as $d \to \infty$.
\end{remark}

\begin{proposition}[Moments of $\log v^2$ in the subcritical regime]
  \label{cor:sync_log_stationary}
  Assume $r < r_c(d)$. Under $\pi^{\bv}$ (the stationary law of $\Set{v_i^2}$ on $(0, \infty)$ from \Cref{thm:sync_v_convergence_d}~(ii)), write $L := \log v^2$; then $\Expsubidx{\pi^{\bv}}{|L|} < \infty$, and
  \begin{itemize}
    \item[(a)] \emph{Stationarity identity.} $\Expsubidx{\pi^{\bv}}{H(L)} \;=\; -\psi(d/2) - \log 2$, where $H(L) := \log\Brack{F(e^L)/e^L}$ and $F(u) = 2(1 - e^{-u/(2r^2)})$.
    \item[(b)] \emph{Jensen upper bound.} $\Expsubidx{\pi^{\bv}}{\log v^2} \;\le\; L_*(r, d)$, where $L_*$ is the unique root of $H(L_*) = -\psi(d/2) - \log 2$; explicitly
    \begin{equation}
      \label{eq:sync_L_star}
      L_*(r, d) \;=\; \log\Brack{2 r^2 \alpha},
    \end{equation}
    with $\alpha \in (0,\infty)$ the unique solution of
    \begin{equation*}
      \frac{1 - e^{-\alpha}}{\alpha} \;=\; \Brack{r / r_c(d)}^2.
    \end{equation*}
    $L_*(r, d) \to -\infty$ as $r \uparrow r_c(d)$; in particular $\Expsubidx{\pi^{\bv}}{\log v^2} \to -\infty$ at criticality.
  \end{itemize}
\end{proposition}

\begin{corollary}[Dimension factorization of $L_*$ at matched subcriticality]
  \label{cor:sync_subcrit_d_scaling}
  Fix $\lambda \in (0, 1)$ and set $r = \lambda\, r_c(d)$. The Jensen bound of \Cref{cor:sync_log_stationary}~(b) factorizes as
  \begin{eqnarray*}
    L_*(\lambda\,r_c(d), d) \;=\; \psi(d/2) + 2\log 2 + \log\Brack{\lambda^2\,\alpha(\lambda)},
  \end{eqnarray*}
  where $\alpha(\lambda) \in (0, \infty)$ is the unique solution of $(1 - e^{-\alpha(\lambda)})/\alpha(\lambda) = \lambda^2$ (universal in $d$).
\end{corollary}

\begin{corollary}[Persistence of coupling: comparison with independent Gaussians in $\RR^d$]
  \label{rem:sync_indep_gaussian_comparison}
  Consider a null model in which a pair of coordinates has the same marginals as our process $\bZ_{i,t}$ but is independent: $\tilde Z_1, \tilde Z_2 \sim \mcN(0, I_d)$ and $\tilde Z_1 \perp \tilde Z_2$. Then $\tilde v^2 = \norm{\tilde Z_1 - \tilde Z_2}^2 \sim 2\chi^2_d$ and
  \begin{eqnarray}
    \label{eq:sync_indep_gaussian_logv}
    \Exp{\log \tilde v^2} \;=\; \psi(d/2) + 2\log 2 \;=\; \log\Brack{2\,r_c(d)^2}.
  \end{eqnarray}
  In the notation of \Cref{cor:sync_subcrit_d_scaling}, at $r = \lambda\,r_c(d)$ with $\lambda \in (0, 1)$, the Jensen bound sits \emph{strictly below} this benchmark:
  \begin{eqnarray}
    \label{eq:sync_indep_gaussian_deficit}
    L_*(\lambda\,r_c(d), d) \;-\; \Exp{\log \tilde v^2} \;=\; \log\Brack{\lambda^2\,\alpha(\lambda)} \;<\; 0.
  \end{eqnarray}
  The deficit $\log[\lambda^2 \alpha(\lambda)]$ is \emph{independent of $d$} and  diverges to $-\infty$ as $\lambda \uparrow 1$ (critical threshold).

  A non-zero deficit witnesses dependence directly: the marginals $\bZ_{\infty,1}, \bZ_{\infty,2} \sim \mcN(0, I_d)$ are fixed by the chain, so under independence one would have $v^2 \sim 2\chi^2_d$ exactly and $\Exp{\log v^2} = \log[2\,r_c(d)^2]$. The strict deficit $\log[\lambda^2\,\alpha(\lambda)] < 0$ therefore rules out independence between $\bZ_{\infty,1}$ and $\bZ_{\infty,2}$. In particular, for any $d$, the dependence can be made arbitrarily strong by taking $\lambda \in (0, 1)$ large enough.
\end{corollary}

We observe that, as expected, by taking $\lambda \uparrow 1$ in \eqref{eq:sync_indep_gaussian_deficit} we approach the supercritical full-dependence regime, where $\bZ_{\infty,1} = \bZ_{\infty,2}$ and $\Exp{\log v^2} = -\infty$.

The strict negativity in \eqref{eq:sync_indep_gaussian_deficit} is a consequence of the defining equation for $\alpha(\lambda)$: multiplying $(1 - e^{-\alpha(\lambda)})/\alpha(\lambda) = \lambda^2$ by $\alpha(\lambda)$ gives $\lambda^2\, \alpha(\lambda) = 1 - e^{-\alpha(\lambda)} < 1$.

\section{Experiments}
\label{sec:experiments}

In this section we verify the threshold $r_c(d)$ empirically across $d \in \Set{1, 10, 100}$, and then use the bounds of \Cref{sec:results} to choose $\lambda$ per dimension and explore the unstructured (Gaussian-baseline) and structured (chain-specific multimodal) regimes of $\pi_{\bZ}$. Code to reproduce the experiments is available at \url{https://github.com/komarkdev/deep-gp-thresholds}.

For the threshold check we simulate the scalar pairwise-distance chain \eqref{eq:kernel_corr_converegnce_d} in log-space ($L_i = \log v_i^2$) to avoid spurious absorption at $0$ when $v_i$ becomes moderately small. Throughout we parametrize $r = \lambda\, r_c(d)$, so the asymptotic rate of decay of $L_i$, $-\rho_d = 2\log r - \psi(d/2) - \log 2$ (see \Cref{thm:sync_v_convergence_d}~(i)), equals $-2\log\lambda$ \emph{independently of $d$}.

\Cref{fig:exp_d100_main} confirms the dichotomy of \Cref{thm:sync_v_convergence_d} at $d = 100$; the analogous figures at $d = 1, 10$ are provided in supplementary \Cref{sec:appendix_more_figs}, \Cref{fig:appendix_trajectories_above,fig:appendix_trajectories_below,fig:appendix_stationary_multidim}, and exhibit a generally similar behaviour. Above $r_c(d)$ (left pane of \Cref{fig:exp_d100_main}), the mean of $\log v_i^2$ decays linearly at the predicted rate $-2\log\lambda$, with the dashed reference slope hidden behind the simulated curve. Below $r_c(d)$ (middle pane of \Cref{fig:exp_d100_main}), the mean of $\log v_i^2$ stabilises at a fixed value at every $\lambda$ we tested except the most near-critical, consistent with convergence; the stabilisation slows as $\lambda \uparrow 1$, and at $\lambda = 0.999$ the trajectory is still drifting at depth $2000$. The stationary law (right pane of \Cref{fig:exp_d100_main}) acquires a heavier left tail as $\lambda \uparrow 1$, with the Jensen bound $L_*$ nearly tight at $d = 100$. At low $d$ the bound becomes loose: at $d = 1, \lambda = 0.94$ the empirical mean sits markedly below $L_*$ (Supplementary \Cref{fig:appendix_stationary_multidim}).

\begin{figure}[t]
  \centering
  \includegraphics[width=\textwidth]{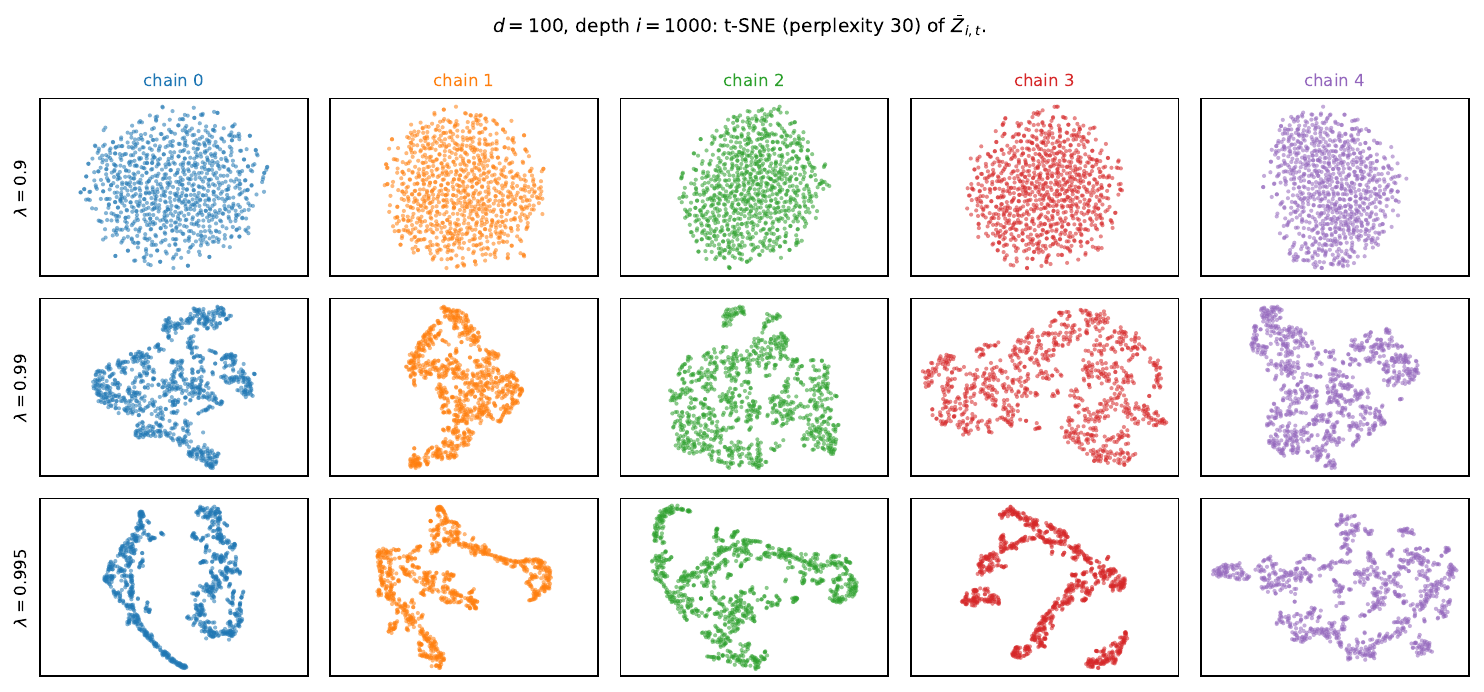}
  \caption{Per-chain t-SNE embeddings of $\bZ_{i,t} \in \RR^{100}$ at depth $i = 1000$, for three subcritical $\lambda$ (rows) and five i.i.d.\ chains (columns); $n = 1000$ points per panel, perplexity $30$, PCA initialisation. The visible-structure window in $d = 100$ is narrow: $\lambda = 0.90$ (top row) embeds as a featureless isotropic blob in every chain, indistinguishable across chains; $\lambda = 0.99$ and $\lambda = 0.995$ (lower rows) produce chain-specific multimodal structure --- the non-Gaussianity of $\pi_{\bZ}$ from \Cref{prop:sync_not_gaussian} made visible. The same chain-specific structure is recovered by linear PCA (\Cref{fig:appendix_pi_Z_pca_d100} in \Cref{sec:appendix_more_figs}), so the patterns here are not an artefact of the t-SNE embedding.}
  \label{fig:exp_pi_Z_tsne_main}
\end{figure}

\paragraph{From bounds to a choice of $\lambda$.}
The theory of \Cref{sec:results} gives more than a sharp threshold: it tells us how to choose $\lambda$ per dimension to see informative structure under $\pi_{\bZ}$. The relevant tool is the i.i.d.-Gaussian comparison of \Cref{rem:sync_indep_gaussian_comparison}: at $r = \lambda\, r_c(d)$, the stationary mean of $\log v^2$ sits below its i.i.d.-Gaussian benchmark by at least $\Abs{\log[\lambda^2\,\alpha(\lambda)]} > 0$ --- the Jensen bound $L_*$ sits below the benchmark by exactly this amount (see eq.~\eqref{eq:sync_indep_gaussian_deficit}), and the stationary mean lies at or below $L_*$ --- which is $d$-independent at fixed $\lambda$. The benchmark itself, $\log[2\, r_c(d)^2] \sim \log d$, grows with dimension, so for $\bZ$ to look distinguishable from the independent-Gaussian null we would like the deficit to be of similar order. This is what fixes the appropriate $\lambda$ at each dimension: a moderate $\lambda$ suffices at small $d$, but at large $d$, $\alpha(\lambda)$ has to shrink to make the deficit $\log d$-large, forcing $\lambda$ close to $1$. Concretely, the defining equation $(1 - e^{-\alpha(\lambda)})/\alpha(\lambda) = \lambda^2$ gives $\alpha(\lambda) \approx 2(1 - \lambda^2)$ as $\lambda \uparrow 1$, and hence
\begin{equation}
  \label{eq:deficit_asymptotic}
  \log[\lambda^2\,\alpha(\lambda)] \;\approx\; \log[2(1 - \lambda^2)].
\end{equation}
The guaranteed deficit therefore reaches order $-\log d$ at $1 - \lambda^2 \sim 1/d$, i.e.\ $\lambda \approx 1 - 1/(2d)$ at large $d$. This estimate agrees with the experiments below: the visible-structure window we observe at $d = 100$ sits at $\lambda \in (0.99, 1)$ and at $d = 10$ at $\lambda \in (0.91, 0.97)$, both consistent with $1 - \lambda \asymp 1/d$.

\paragraph{Sampling from $\pi_{\bZ}$.}
We simulate the position chain $\bZ_i \in (\RR^d)^n$ for $n = 1000$ and visualise the resulting cloud one chain at a time, with $\lambda$ values chosen per dimension to span the regime transition from a near-independent Gaussian baseline to chain-specific multimodal structure: $d = 1$ at $\lambda \in \Set{0.10, 0.30, 0.60, 0.85}$, $d = 10$ at $\lambda \in \Set{0.66, 0.91, 0.95, 0.97}$, $d = 100$ at $\lambda \in \Set{0.90, 0.99, 0.995}$. Five i.i.d.\ chains per $\lambda$. \Cref{fig:exp_pi_Z_tsne_main} shows the $d = 100$ case as a per-chain t-SNE embedding at depth $i = 1000$. Note that according to \Cref{fig:exp_d100_main}, for the $\lambda$ in the above range, at that depth the chains have already stabilised. At the smallest $\lambda$ ($\lambda = 0.90$, top row) the cloud is an approximately isotropic blob and the chains are visually indistinguishable from one another; as $\lambda$ approaches $1$ the embeddings fragment into chain-specific multimodal structure --- clusters, strings, loops --- in the narrow strip $\lambda \in (0.99, 1)$, with the structure becoming more pronounced for higher $\lambda$. The structure and the run-to-run variation of the chains visualise the non-Gaussianity of $\pi_{\bZ}$ established in \Cref{prop:sync_not_gaussian}: a Gaussian limit would force every chain to look like a single common ellipsoidal cloud. Note that the structure visible at large $\lambda$ is not an artefact of the embedding: all the figures use the same (default) t-SNE parameters, and at small $\lambda$ the same t-SNE pipeline returns an essentially featureless blob. Moreover, the corresponding linear PCA view in \Cref{sec:appendix_more_figs} (\Cref{fig:appendix_pi_Z_pca_d100}) shows similar chain-specific bananas, U-shapes, and arcs without any non-linear embedding step. The analogous figures at $d = 1$ (per-chain histograms) and $d = 10$ (PCA and t-SNE, where the visible-structure window is wider and the transition smoother) are reported in \Cref{sec:appendix_more_figs}; as with the $d = 100$ case, at $d = 1, 10$ the chain-specific structure becomes more pronounced as $\lambda$ enters its $d$-dependent near-critical strip.

\section{Conclusions and Future Work}
\label{sec:conclusions}

We have established the existence and several basic properties of the depth-infinite limit $\bZ_{\infty}$ of the compositional-GP prior chain $\bZ_i$. For the RBF kernel we identified a sharp threshold $r_c(d) = \sqrt{2}\, e^{\psi(d/2)/2}$ separating degeneration into a constant function from survival of a nontrivial limit; below the threshold we showed that $\bZ_{\infty}$ exists, is not Gaussian, and its components exhibit nontrivial dependence, so that the prior retains genuine structure at depth.

Several directions remain open. First, while we have established the existence and uniqueness of the limit $\pi_{\bZ}$, the \emph{rate} of convergence $\Law(\bZ_i) \to \pi_{\bZ}$ remains open; establishing quantitative, and ideally uniform, rates --- with prefactors independent of the initial state $\bZ_1$ --- would be a natural next step and is a prerequisite for any finite-depth application. The question connects to a substantial literature on quantitative ergodicity for iterated random functions \citep{diaconis1999iterated,wushao2004limit,stenflo2012survey}, though the infinite-dimensional GP setting we consider falls outside the parametric families that line of work has typically addressed.

Second, the structure of $\pi_{\bZ}$ itself is largely unexplored beyond the basic properties we establish here, and analysing its dependence structure in more detail is a natural next step. A concrete and likely tractable question is whether $\pi_{\bZ}$ belongs to any known family of probability distributions, or constitutes a genuinely new family --- the empirical multimodality observed in \Cref{sec:experiments} suggests it is not a simple modification of a familiar parametric class.

Third, 
the arguments we use here appear not to apply at the critical bandwidth value $r = r_c(d)$ itself, and it will be interesting to understand the behaviour there as well. Our results are also specific to the RBF kernel; whether the same dichotomy --- a sharp threshold separating synchronisation from a non-degenerate limit --- holds for other kernel families (e.g.\ Mat\'ern) is an obvious open problem. The arguments here use the RBF kernel structure both for the global one-step bound $1 - \phi(v) \le v^2/(2r^2)$ underlying the supercritical SLLN and for the boundary linearisation used in the Foster--Lyapunov drift; replacing or generalising these with kernel-agnostic arguments would extend the picture beyond the Gaussian-form case.

\section*{Acknowledgments and Disclosure of Funding}
This research was supported by the Israel Science Foundation (grant No.~2198/25).

\bibliography{refs}
\bibliographystyle{plainnat}

\clearpage
\appendix

\section{Proofs}
\label{sec:proofs}

The dichotomy of \Cref{thm:sync_v_convergence_d} is proved in two stages. We first treat the scalar case $d = 1$ (\Cref{prop:sync_v_convergence} below), where the innovation is $g_i^2$ with $g_i \sim \mcN(0,1)$; all of the analytic work is here. The $\RR^d$ case is then recovered by substituting $X_i \sim \chi^2_d$ for $g_i^2$ throughout: the functional form of the log-chain, the Foster--Lyapunov function, and the structural ingredients (irreducibility, aperiodicity, small sets) are unchanged, and only the noise mean $\Exp{\log g^2} = -\gamma - \log 2$ is replaced by $\Exp{\log X} = \psi(d/2) + \log 2$.

For $d = 1$, the recursion \eqref{eq:kernel_corr_converegnce_d} reduces to
\begin{eqnarray}
  \label{eq:kernel_corr_converegnce}
  v_{i+1} \;=\; \sqrt{2(1 - \phi(v_i))}\, g_i, \qquad g_i \sim \mcN(0, 1) \text{ i.i.d.},
\end{eqnarray}
with $v_i := Z_{i,1} - Z_{i,2} \in \RR$.

\begin{proposition}[Scalar dichotomy, $d = 1$]
  \label{prop:sync_v_convergence}
  Let $\Set{v_i}_{i \ge 1}$ be the chain defined by \eqref{eq:kernel_corr_converegnce} with $v_1 \in \RR$ arbitrary, and let
  \begin{eqnarray*}
    r_c \;:=\; \frac{1}{\sqrt{2}\, e^{\gamma/2}} \;=\; \sqrt{\frac{1}{2 e^\gamma}} \;\approx\; 0.5298.
  \end{eqnarray*}
  \begin{itemize}
    \item[(i)] If $r > r_c$, then for every $v_1 \in \RR$, $v_i \to 0$ almost surely, and
      \begin{eqnarray*}
        \limsup_{i \to \infty} \frac{1}{i}\, \log v_i^2 \;\le\; -\gamma - \log 2 - 2\log r \;<\; 0 \quad \text{a.s.}
      \end{eqnarray*}
    \item[(ii)] If $r < r_c$ and $v_1 \ne 0$, then $v_i \not\to 0$ almost surely, and $\Set{v_i^2}$ admits a unique nontrivial stationary distribution on $(0, \infty)$, independent of $v_1$.
  \end{itemize}
\end{proposition}

\begin{proof}[Proof of \Cref{prop:sync_v_convergence}]
  Set $u_i := v_i^2 \ge 0$, $L_i := \log u_i$, and $F(u) := 2(1 - e^{-u/(2r^2)})$. Then \eqref{eq:kernel_corr_converegnce} reads
  \begin{eqnarray}
    \label{eq:u_chain}
    u_{i+1} \;=\; F(u_i)\,g_i^2.
  \end{eqnarray}
  If $u_1 = 0$ then $u_i \equiv 0$; otherwise $u_i > 0$ a.s.\ for all $i$ (since $g_i \ne 0$ a.s.), and \eqref{eq:u_chain} gives
  \begin{eqnarray}
    \label{eq:L_chain}
    L_{i+1} & = & L_i + H(L_i) + \log g_i^2,
  \end{eqnarray}
  where $H(L) := \log\Brack{F(e^L)/e^L}$.
  The function $H$ is continuous on $\RR$, satisfies $H(L) \to -2\log r$ as $L \to -\infty$ (since $F(u)/u \to 1/r^2$ as $u \to 0$), and obeys the uniform upper bound
  \begin{eqnarray}
    \label{eq:H_upper}
    H(L) \;\le\; -2\log r \qquad \text{for all } L \in \RR,
  \end{eqnarray}
  which follows from $1 - e^{-x} \le x$. For $g \sim \mcN(0,1)$, $\Exp{\log g^2} = \psi(1/2) + \log 2 = -\gamma - \log 2$; set $\rho := -\gamma - \log 2 - 2\log r$, so that $r > r_c$ iff $\rho < 0$, and $r < r_c$ iff $\rho > 0$.

  \emph{Proof of (i).} Assume $r > r_c$, so $\rho < 0$. From \eqref{eq:L_chain} and \eqref{eq:H_upper},
  \begin{eqnarray*}
    L_i \;\le\; L_1 + \sum_{j=1}^{i-1}\Brack{-2\log r + \log g_j^2}.
  \end{eqnarray*}
  We invoke Kolmogorov's strong law of large numbers for i.i.d.\ sequences (Theorem~2.4.1 of \cite{durrett2019probability}): if $\Set{Y_j}_{j \ge 1}$ are i.i.d.\ with $\Exp{|Y_1|} < \infty$, then $n^{-1}\sum_{j=1}^n Y_j \to \Exp{Y_1}$ almost surely. With $Y_j := -2\log r + \log g_j^2$ the sequence is i.i.d.\ (since the $g_j$ are), and the integrability hypothesis $\Exp{|\log g^2|} < \infty$ holds by \Cref{rem:log_g_sq_integrable}. The SLLN therefore yields $(i-1)^{-1}\sum_{j=1}^{i-1}(-2\log r + \log g_j^2) \to \rho$ a.s., whence $\limsup_{i\to\infty} L_i/i \le \rho < 0$ a.s. So $u_i \to 0$ a.s.\ at exponential rate at least $|\rho|$, i.e.\ $v_i \to 0$ a.s.

  \emph{Proof of (ii), non-convergence.} Assume $r < r_c$, so $\rho > 0$. Fix $\eps \in (0, \rho/2)$. Since $H(L) \to -2\log r$ as $L \to -\infty$, there exists $L_* < 0$ such that
  \begin{eqnarray}
    \label{eq:H_lower}
    H(L) \;\ge\; -2\log r - \eps \qquad \text{for all } L \le L_*;
  \end{eqnarray}
  write $\delta := e^{L_*} > 0$. Fix $u_1 > 0$ and define the hitting time
  \begin{eqnarray*}
    \tau \;:=\; \inf\Set{i \ge 1 : u_i > \delta}.
  \end{eqnarray*}
  We claim $\tau < \infty$ a.s. On the event $\Set{\tau = \infty}$, $L_i \le L_*$ for every $i$, so \eqref{eq:L_chain} and \eqref{eq:H_lower} give the pathwise lower bound
  \begin{eqnarray}
    \label{eq:L_lower_on_tau}
    L_{n+1} & \ge & L_1 + \sum_{j=1}^n\Brack{-2\log r - \eps + \log g_j^2} \nonumber \\
    & = & L_1 + n\Brack{\rho - \eps} + S_n,
  \end{eqnarray}
  where $S_n := \sum_{j=1}^n(\log g_j^2 - \Exp{\log g^2})$ is a zero-mean random walk with i.i.d.\ increments. By the same SLLN as in part~(i), applied \emph{unconditionally} to the i.i.d.\ sequence $\Set{\log g_j^2 - \Exp{\log g^2}}_{j \ge 1}$ on the underlying probability space, $\Prob{\Omega_0} = 1$ for $\Omega_0 := \Set{S_n/n \to 0}$; we do \emph{not} condition on $\Set{\tau = \infty}$ before invoking the SLLN, since under that conditioning the $g_j$ would no longer be i.i.d. Now consider any sample point $\omega \in \Set{\tau = \infty} \cap \Omega_0$. The pathwise bound \eqref{eq:L_lower_on_tau} together with $S_n(\omega)/n \to 0$ and $\rho - \eps > \rho/2 > 0$ force $L_{n+1}(\omega) \to +\infty$, contradicting $L_{n+1}(\omega) \le L_*$ for all $n$ (which holds since $\omega \in \Set{\tau = \infty}$). Hence $\Set{\tau = \infty} \cap \Omega_0 = \emptyset$, and combined with $\Prob{\Omega_0} = 1$ this gives $\Prob{\tau = \infty} = 0$.

  Iterate: set $\tau_0 := 0$ and $\tau_{k+1} := \inf\Set{i > \tau_k : u_i > \delta}$. By the strong Markov property applied at $\tau_k + 1$ (from which point the chain restarts from the positive state $u_{\tau_k + 1}$), the same argument gives $\tau_{k+1} < \infty$ a.s.\ on $\Set{\tau_k < \infty}$. By induction every $\tau_k$ is a.s.\ finite, so $u_i > \delta$ for infinitely many $i$ a.s.; in particular
  \begin{eqnarray*}
    \limsup_{i\to\infty} u_i \;\ge\; \delta \;>\; 0 \qquad \text{a.s.,}
  \end{eqnarray*}
  so $v_i = \pm\sqrt{u_i}$ does \emph{not} converge to $0$ a.s.

  \emph{Proof of (ii), stationary distribution.} We establish positive Harris recurrence via a Foster--Lyapunov drift argument; we briefly recall the framework before applying it.

  Let $\Set{X_i}$ be a Markov chain on a state space $\mcX$ with transition kernel $P(x, \cdot)$. The chain is \emph{$\psi$-irreducible} if there exists a non-trivial measure $\psi$ on $\mcX$ such that for every $x \in \mcX$ and every measurable $A$ with $\psi(A) > 0$, the chain reaches $A$ from $x$ with positive probability in some finite number of steps; intuitively, no part of the state space is invisible from any starting point. A set $C \subset \mcX$ is \emph{small} if there exist $m \ge 1$, $\eta > 0$, and a probability measure $\nu$ such that $P^m(x, \cdot) \ge \eta\,\nu(\cdot)$ for every $x \in C$, i.e.\ $C$ is a uniform regeneration region: from anywhere in $C$, the $m$-step law dominates a fixed probability measure. For a $\psi$-irreducible chain on a general state space there is a notion of \emph{period} $d \ge 1$, defined via the $d$-cycle of small sets associated with the chain (see Section~5.4 of \cite{meyn1993markov}, where the period is defined for general state-space $\psi$-irreducible chains), and the chain is \emph{aperiodic} precisely when $d = 1$. A convenient sufficient condition --- which we will use --- is that some small set $C$ admit a one-step minorisation $P(x, \cdot) \ge \eta\,\nu$ for every $x \in C$ (i.e.\ the minorising step $m$ above can be taken equal to $1$); a chain satisfying this is called \emph{strongly aperiodic}, and strong aperiodicity implies aperiodicity.

  Two results from \cite{meyn1993markov} combine to give positive recurrence and convergence to equilibrium under a Lyapunov drift condition. First, \emph{Foster's drift criterion} (Theorem~11.3.4 of \cite{meyn1993markov}): if the chain is $\psi$-irreducible and there exist a measurable function $V : \mcX \to [0, \infty)$, a small set $C$, and constants $c > 0$, $b < \infty$, such that the drift condition
  \begin{eqnarray}
    \label{eq:foster_lyapunov_drift}
    \Exp{V(X_{i+1}) \cond X_i = x} \;\le\; V(x) - c + b\,\mathbf{1}_C(x)
  \end{eqnarray}
  holds for every $x \in \mcX$, then the chain is positive Harris recurrent and admits a unique invariant probability measure $\pi$. Second, the \emph{Aperiodic Ergodic Theorem} (Theorem~13.0.1 of \cite{meyn1993markov}): once the chain is positive Harris recurrent, aperiodic, and has invariant probability $\pi$, the law of $X_i$ converges to $\pi$ in total variation from every starting point $x \in \mcX$,
  \begin{eqnarray*}
    \norm{P^i(x, \cdot) - \pi}_{TV} \to 0 \qquad \text{as } i \to \infty.
  \end{eqnarray*}
  The function $V$ plays the role of a Lyapunov function: it measures a generalised ``distance'' to the small set $C$, and \eqref{eq:foster_lyapunov_drift} asserts that this distance shrinks on average by at least $c$ outside $C$, while inside $C$ it may grow but only by a bounded amount $b$. Together these prevent both escape to infinity and trapping at the boundary of $\mcX$, forcing the chain to return to $C$ infinitely often with controlled hitting times --- which is exactly what positive recurrence requires.

  We apply this framework to the chain $\Set{u_i}$ on $\mcX = (0, \infty)$ with transition $u \mapsto F(u)\,g^2$. Since $g^2$ has a smooth strictly positive Lebesgue density on $(0,\infty)$ and $F(u) > 0$ for $u > 0$, the one-step transition law of $u_{i+1}$ given $u_i = u$ has a strictly positive Lebesgue density on $(0,\infty)$, smoothly depending on $u$. This has three consequences. \emph{(a)} The chain is $\psi$-irreducible with respect to Lebesgue measure on $(0, \infty)$: any set of positive Lebesgue measure has positive one-step transition probability from every $u \in (0,\infty)$. \emph{(b)} Every compact $K \subset (0,\infty)$ is a small set with $m = 1$: the one-step density is bounded below by some $\eta_K > 0$ uniformly on $K \times K$ by continuity, hence $P(u, \cdot) \ge \eta_K\,|K|\cdot \mathrm{Unif}(K)$ for every $u \in K$. \emph{(c)} The chain is aperiodic: $m = 1$ in (b) gives the strong aperiodicity condition. Since $L = \log u$ is a homeomorphism between $(0,\infty)$ and $\RR$, all of these properties transfer verbatim to the log-chain $\Set{L_i}$ on $\RR$ (with Lebesgue measure pulled back to Lebesgue on $\RR$).

  For the Lyapunov function we work in log-coordinates, where the multiplicative dynamics become additive (cf.\ \eqref{eq:L_chain}). Define
  \begin{eqnarray*}
    V(L) \;:=\; (L_0 - L)^+ + (L - L^0)^+,
  \end{eqnarray*}
  where $L_0 > L_*$ and $L^0 > L_0$ are to be fixed. This is a ``tent'' function: $V \equiv 0$ on the strip $L \in [L_0, L^0]$, and $V$ grows linearly outside it (with slope $-1$ to the left of $L_0$ and slope $+1$ to the right of $L^0$). Geometrically, $V(L)$ is the distance from $L$ to the strip $[L_0, L^0]$, which in $u$-coordinates is the compact interval $[e^{L_0}, e^{L^0}] \subset (0,\infty)$ --- a small set by the previous paragraph. Verifying \eqref{eq:foster_lyapunov_drift} thus amounts to showing that $\Exp{V(L_{i+1}) \cond L_i = L} < V(L) - c$ for $L$ in the two tails $\Set{L \le L_*}$ and $\Set{L \ge L^0 + 2\kappa}$, with the in-between region $C := [L_*, L^0 + 2\kappa]$ acting as the small set on which $V$ is allowed to grow by the constant $b$. The mechanisms in the two tails are different: on the \emph{left} ($L \le L_*$) the linearisation-based positive drift $\rho - \eps$ on $L$ from \eqref{eq:H_lower} pushes $L_{i+1}$ rightward, decreasing $V$; on the \emph{right} ($L \ge L^0$) the global cap $u_{i+1} \le 2 g_i^2$ resets $L_{i+1}$ to a state-independent random level of bounded mean, which sits far below $L^0$ once $L^0$ is large, again decreasing $V$. We verify these now.

  \emph{Region 1: $L \le L_*$.} Here $V(L) = L_0 - L$. Using $(a)^+ = a + (-a)^+$,
  \begin{eqnarray*}
    \lefteqn{\Exp{(L_0 - L_{i+1})^+ \cond L_i = L}} \\
    & = & L_0 - \Exp{L_{i+1} \cond L_i = L} \\
    & & {}+ \Exp{(L_{i+1} - L_0)^+ \cond L_i = L}.
  \end{eqnarray*}
  By \eqref{eq:L_chain} and \eqref{eq:H_lower}, $\Exp{L_{i+1} \cond L_i = L} = L + H(L) - \gamma - \log 2 \ge L + (\rho - \eps)$. By \eqref{eq:H_upper}, $L_{i+1} \le L - 2\log r + \log g_i^2$, so for $L \le L_*$,
  \begin{eqnarray*}
    \Exp{(L_{i+1} - L_0)^+ \cond L_i = L} & \le & \Exp{(\log g^2 - M)^+},
  \end{eqnarray*}
  where $M := L_0 - L_* + 2\log r$.
  Since $\log g^2$ has finite mean, $\Exp{(\log g^2 - M)^+} \to 0$ as $M \to \infty$; choose $L_0$ large enough that this is $\le (\rho - \eps)/4$, and by the same bound applied with $L^0$ in place of $L_0$, also $\Exp{(L_{i+1} - L^0)^+ \cond L_i = L} \le (\rho - \eps)/4$ for $L \le L_*$. Using the identity $(L_0 - L_{i+1})^+ = (L_0 - L_{i+1}) + (L_{i+1} - L_0)^+$ together with $L_0 - \Exp{L_{i+1} \cond L_i = L} \le L_0 - L - (\rho - \eps) = V(L) - (\rho - \eps)$ (valid since $V(L) = L_0 - L$ for $L \le L_*$), we obtain
  \begin{eqnarray*}
    \lefteqn{\Exp{(L_0 - L_{i+1})^+ \cond L_i = L}} \\
    & \le & V(L) - (\rho - \eps) + (\rho - \eps)/4 \\
    & = & V(L) - 3(\rho - \eps)/4.
  \end{eqnarray*}
  Adding $\Exp{(L_{i+1} - L^0)^+ \cond L_i = L} \le (\rho - \eps)/4$ gives the drift bound
  \begin{eqnarray}
    \label{eq:drift_left}
    \Exp{V(L_{i+1}) \cond L_i = L} & \le & V(L) - (\rho - \eps)/2
  \end{eqnarray}
  for all $L \le L_*$, i.e.\ $c_1 := (\rho - \eps)/2 > 0$ on $\Set{L \le L_*}$.

  \emph{Region 2: $L \ge L^0$.} Here $V(L) = L - L^0$. From \eqref{eq:u_chain}, $u_{i+1} \le 2 g_i^2$, so $L_{i+1} \le \log 2 + \log g_i^2$ regardless of $L_i$. On Region 2 the lower bound $u_{i+1} \ge F(e^{L^0})\, g_i^2$ also holds (since $F$ is increasing and $u_i \ge e^{L^0}$), giving $L_{i+1} \ge \log F(e^{L^0}) + \log g_i^2$. Combining, for every $L_i \ge L^0$ the new state satisfies
  \begin{eqnarray*}
    \Abs{L_{i+1}} \;\le\; \max\!\Brack{\log 2,\, -\log F(e^{L^0})} + \Abs{\log g_i^2}
    \;=:\; W_i,
  \end{eqnarray*}
  an envelope independent of $L_i$ with finite mean (since $\Exp{|\log g^2|} < \infty$ by \Cref{rem:log_g_sq_integrable}). Hence
  \begin{eqnarray*}
    V(L_{i+1}) \;\le\; \Abs{L_0 - L_{i+1}} + \Abs{L_{i+1} - L^0} \;\le\; 2 W_i + L_0 + L^0,
  \end{eqnarray*}
  so that
  \begin{eqnarray*}
    \kappa \;:=\; \Exp{2 W_i + L_0 + L^0} \;<\; \infty,
  \end{eqnarray*}
  and $\Exp{V(L_{i+1}) \cond L_i = L} \le \kappa$ for every $L \ge L^0$. The constant $\kappa$ is fixed by the chain and the choices of $L_0, L^0$.

  Now we use this $\kappa$ to delimit the small set. We want a strictly negative drift $\le -\kappa$ in Region 2, i.e.\ $\Exp{V(L_{i+1}) \cond L_i = L} \le V(L) - \kappa$, which (using the uniform bound above) holds whenever $V(L) \ge 2\kappa$. Since $V(L) = L - L^0$ in Region 2, this requires $L \ge L^0 + 2\kappa$ --- and this is precisely why the small set is taken to extend up to $L^0 + 2\kappa$ rather than just to $L^0$: the strip $[L^0, L^0 + 2\kappa]$ is a buffer in which $V(L)$ is too small (between $0$ and $2\kappa$) for the uniform-envelope bound to dominate it, so we cannot yet claim a useful negative drift there. Outside this buffer, on $\Set{L \ge L^0 + 2\kappa}$, we have the drift $\Exp{V(L_{i+1}) \cond L_i = L} \le \kappa \le V(L) - \kappa$.

  \emph{Region 3: the compact set $C := [L_*, L^0 + 2\kappa]$.} Here $V(L) \le \max(L_0 - L_*,\, 2\kappa)$. The map $L \mapsto \Exp{V(L_{i+1}) \cond L_i = L}$ is continuous in $L$ (the transition density is continuous and has finite first moment of $|L_{i+1}|$ locally in $L$, since $L_{i+1} = \log F(e^L) + \log g^2$ with $\log F(e^L)$ continuous and $\Exp{|\log g^2|} < \infty$), hence bounded on the compact $C$ by some $b < \infty$. On $C$ we make no claim of a negative drift; this is the small set on which $V$ is allowed to grow, contributing the term $b\,\mathbf{1}_C(L)$ in \eqref{eq:foster_lyapunov_drift}.

  Combining the three regions yields the drift condition
  \begin{eqnarray*}
    \Exp{V(L_{i+1}) \cond L_i = L} \;\le\; V(L) - c + b\,\mathbf{1}_C(L),
  \end{eqnarray*}
  for some $c > 0$ and $b < \infty$, exactly the form \eqref{eq:foster_lyapunov_drift}, with small set $C = [L_*, L^0 + 2\kappa]$. Together with the $\psi$-irreducibility verified above, Foster's drift criterion (Theorem~11.3.4 of \cite{meyn1993markov}) yields that $\Set{u_i}$ is positive Harris recurrent and admits a unique invariant probability measure $\pi$ on $(0, \infty)$. Aperiodicity was also verified above ($m = 1$ minorisation on every compact). The Aperiodic Ergodic Theorem (Theorem~13.0.1 of \cite{meyn1993markov}) then applies, giving that the law of $u_i$ converges to $\pi$ in total variation from every $u_1 > 0$; in particular $\pi$ is independent of $v_1$. Non-degeneracy $\pi((\delta,\infty)) > 0$ follows from the non-convergence step (the chain spends a positive fraction of time above $\delta$).
\end{proof}

\begin{remark}
  \label{rem:log_g_sq_integrable}
  For $g \sim \mcN(0, 1)$, $\Exp{|\log g^2|} < \infty$ (used in the proof of (i) and in the non-convergence step of (ii)). Indeed, $g^2 \sim \chi^2_1$ has density $f(x) = (2\pi x)^{-1/2} e^{-x/2}$ on $(0, \infty)$, and
  \begin{eqnarray*}
    \Exp{|\log g^2|} \;=\; \int_0^\infty |\log x|\,(2\pi x)^{-1/2} e^{-x/2}\,dx.
  \end{eqnarray*}
  Split at $x = 1$. The upper tail $\int_1^\infty (\log x)(2\pi x)^{-1/2} e^{-x/2}\,dx$ is finite because $(\log x)\,x^{-1/2} e^{-x/2}$ decays exponentially. For the lower tail, $|\log x| = -\log x$ on $(0,1)$ and $e^{-x/2} \le 1$, so
  \begin{eqnarray*}
    \int_0^1 (-\log x)\,(2\pi x)^{-1/2} e^{-x/2}\,dx \;\le\; (2\pi)^{-1/2}\!\int_0^1 \frac{-\log x}{\sqrt{x}}\,dx.
  \end{eqnarray*}
  Substitute $u := -\log x$, so $x = e^{-u}$, $dx = -e^{-u}\,du$, $\sqrt{x} = e^{-u/2}$, and the limits $x \in (0,1)$ map to $u \in (\infty, 0)$. The integrand transforms as
  \begin{eqnarray*}
    \frac{-\log x}{\sqrt{x}}\,dx \;=\; \frac{u}{e^{-u/2}}\cdot(-e^{-u}\,du) \;=\; -u\,e^{-u/2}\,du,
  \end{eqnarray*}
  and the sign flip cancels against the reversed limits, yielding
  \begin{eqnarray*}
    \int_0^1\frac{-\log x}{\sqrt{x}}\,dx \;=\; \int_0^\infty u\,e^{-u/2}\,du \;=\; 4.
  \end{eqnarray*}
  The substitution thus replaces the logarithmic singularity at $x = 0$ by a linear factor $u$ against an exponentially decaying weight, which is harmless. Combined with finiteness of the upper tail, this gives $\Exp{|\log g^2|} < \infty$, and a fortiori $\Exp{\log g^2}$ is well-defined and equals $-\gamma - \log 2$.
\end{remark}

\begin{proof}[Proof of \Cref{thm:sync_v_convergence_d}]
  \emph{Reduction to \eqref{eq:kernel_corr_converegnce_d}.} Conditional on $\bZ_i$, the $k$-th coordinate $V_{i+1}^{(k)} = G_i^{(k)}(\bZ_{i,1}) - G_i^{(k)}(\bZ_{i,2})$ is centred Gaussian with variance $2(1 - \phi(v_i))$, since $G_i^{(k)}(z) \sim \mcN(0,1)$ pointwise and $\Cov(G_i^{(k)}(\bZ_{i,1}), G_i^{(k)}(\bZ_{i,2})) = \phi(\norm{\bZ_{i,1} - \bZ_{i,2}}) = \phi(v_i)$. The coordinates are independent across $k$ because the GPs $G_i^{(1)},\ldots,G_i^{(d)}$ are independent. Hence
  \begin{eqnarray*}
    V_{i+1} \;=\; \sqrt{2(1-\phi(v_i))}\,\bg_i, \qquad \bg_i \sim \mcN(0, I_d) \text{ ind.\ of $v_i$,}
  \end{eqnarray*}
  and $v_{i+1}^2 = 2(1-\phi(v_i))\,\norm{\bg_i}^2$ with $\norm{\bg_i}^2 \sim \chi^2_d$. By rotation invariance the conditional law of $v_{i+1}^2$ given $\bZ_i$ depends only on $v_i$, so $\Set{v_i^2}$ is Markov with transition law \eqref{eq:kernel_corr_converegnce_d}.

  \emph{Dichotomy.} The proof of \Cref{prop:sync_v_convergence} extends to the $\RR^d$ chain by substituting $X_i \sim \chi^2_d$ for $g_i^2$ throughout. The functional ingredients of the argument --- the map $F(u) = 2(1 - e^{-u/(2r^2)})$, the log-chain function $H(L) = \log(F(e^L)/e^L)$, the uniform upper bound \eqref{eq:H_upper}, the boundary linearisation $H(L) \to -2\log r$ as $L \to -\infty$, and the Lyapunov function $V$ --- depend only on the recursion \eqref{eq:kernel_corr_converegnce_d} and not on the specific law of the noise, so they remain unchanged. The only quantity that changes is the noise mean: for $X \sim \chi^2_d = 2\,\mathrm{Gamma}(d/2,1)$,
  \begin{eqnarray*}
    \Exp{\log X} \;=\; \psi(d/2) + \log 2,
  \end{eqnarray*}
  finite for every $d \ge 1$. Define $\rho_d := \psi(d/2) + \log 2 - 2\log r$, so $r > r_c(d) \iff \rho_d < 0$ and $r < r_c(d) \iff \rho_d > 0$. The pointwise bound $H(L) \le -2\log r$ from \eqref{eq:H_upper} is unchanged (it uses only $1 - e^{-x} \le x$); the limit $H(L) \to -2\log r$ as $L \to -\infty$ is unchanged. The SLLN argument of part (i) goes through verbatim with $\log g_j^2$ replaced by $\log X_j$ and $\Exp{\log g^2}$ replaced by $\Exp{\log X}$, giving $\limsup_i L_i / i \le \rho_d$ a.s. The non-convergence step in part (ii) is identical (zero-mean random walk $S_n := \sum_{j=1}^n (\log X_j - \Exp{\log X})$ obeys $S_n / n \to 0$ a.s.\ by the SLLN, and the strong-Markov iteration is unchanged).

  For the Foster--Lyapunov step in part (ii): the transition law of $\Set{u_i = v_i^2}$ given $u_i = u > 0$ is the law of $2(1-\phi(\sqrt u))\,X$ with $X \sim \chi^2_d$. The $\chi^2_d$ density $\frac{1}{2^{d/2}\Gamma(d/2)} x^{d/2-1} e^{-x/2}$ is strictly positive and continuous on $(0,\infty)$ for every $d \ge 1$, so this transition law has a strictly positive Lebesgue density on $(0,\infty)$, smoothly depending on $u$. The same three-fold consequence as in the proof of \Cref{prop:sync_v_convergence} holds: \emph{(a)} the chain $\Set{u_i}$ is $\psi$-irreducible w.r.t.\ Lebesgue measure on $(0,\infty)$; \emph{(b)} every compact subset of $(0,\infty)$ is small with $m = 1$; \emph{(c)} the chain is strongly aperiodic. The Lyapunov function $V(L) = (L_0 - L)^+ + (L - L^0)^+$ from the proof of \Cref{prop:sync_v_convergence} works unchanged, with $g_i^2$ replaced by $X_i \sim \chi^2_d$ in every drift computation; the only input needed is $\Exp{(\log X - M)^+} \to 0$ as $M \to \infty$, which holds because $\Exp{|\log X|} < \infty$ (the verification is identical to \Cref{rem:log_g_sq_integrable}, using the $\chi^2_d$ density in place of $\chi^2_1$). The drift condition \eqref{eq:foster_lyapunov_drift} is therefore satisfied with the same small set $C = [L_*, L^0 + 2\kappa]$ (where now $\kappa$ depends on the $\chi^2_d$ envelope), and Foster's drift criterion (Theorem~11.3.4 of \cite{meyn1993markov}) yields that $\Set{u_i}$ is positive Harris recurrent with a unique invariant probability $\pi^{\bv}$ on $(0,\infty)$. The Aperiodic Ergodic Theorem (Theorem~13.0.1 of \cite{meyn1993markov}), whose hypotheses (positive Harris recurrence, aperiodicity, invariant probability) are now in place, then gives convergence of the law of $u_i$ to $\pi^{\bv}$ in total variation from every $u_1 > 0$; in particular $\pi^{\bv}$ is independent of $v_1$.
\end{proof}

\begin{proof}[Proof of \Cref{thm:sync_Z_convergence}]
  The proof has two parts: (I) ergodicity of the pairwise-distance chain via a sum-of-pairs Lyapunov function, and (II) transfer from $\bu^{(i)}$ to $\bZ_i$ via a mixture bound.

  \emph{Part I: ergodicity of $\Set{\bu^{(i)}}$.}

  \emph{I.1: Markov property and marginal structure.} Conditional on $\bZ_i$, $\bZ_{i+1} \sim \mcN(0, R(\bZ_i) \otimes I_d)$ with $R(\bZ_i)_{st} = \phi(\sqrt{u_{st}^{(i)}})$, a function of $\bu^{(i)}$ alone. Hence $\bu^{(i+1)}$ is a function of $\bZ_{i+1}$ whose conditional distribution given $\bZ_i$ depends on $\bZ_i$ only through $\bu^{(i)}$, so $\Set{\bu^{(i)}}$ is Markov on $\mcU_{n,d}$. For each fixed pair $(s,t)$,
  \begin{eqnarray*}
    V_{st}^{(i+1)} := Z_{i+1,s} - Z_{i+1,t} \cond \bZ_i \;\sim\; \mcN\!\bigl(0,\, 2(1 - \phi(\sqrt{u_{st}^{(i)}}))\, I_d\bigr),
  \end{eqnarray*}
  because the $d$ output coordinates of $G_i$ are i.i.d.\ and the $(s,t)$-marginal of $R(\bZ_i) \otimes I_d$ contributes variance $2(1 - \phi(\sqrt{u_{st}^{(i)}}))$. Consequently
  \begin{eqnarray*}
    u_{st}^{(i+1)} \;=\; 2\bigl(1 - \phi(\sqrt{u_{st}^{(i)}})\bigr)\, X_{st}^{(i)}, \qquad X_{st}^{(i)} := \norm{V_{st}^{(i+1)}}^2 / \bigl(2(1 - \phi(\sqrt{u_{st}^{(i)}}))\bigr) \sim \chi^2_d \ \text{marginally},
  \end{eqnarray*}
  the same recursion as the scalar chain \eqref{eq:kernel_corr_converegnce_d}. Although $\Set{X_{st}^{(i)}}_{s<t}$ are jointly correlated (they share the single GP sample $G_i$), each is marginally $\chi^2_d$, and in particular the \emph{marginal} conditional distribution of $u_{st}^{(i+1)}$ given $\bu^{(i)}$ depends only on $u_{st}^{(i)}$ and is identical to the scalar transition kernel.

  \emph{I.2: Per-pair quadratic drift in log-coordinates.} Work in the log-chain $L_{st}^{(i)} := \log u_{st}^{(i)}$. By \emph{I.1}, the marginal recursion is
  \begin{eqnarray*}
    L_{st}^{(i+1)} \;=\; L_{st}^{(i)} + H(L_{st}^{(i)}) + \log X_{st}^{(i)}, \qquad X_{st}^{(i)} \sim \chi^2_d \ \text{marginally,}
  \end{eqnarray*}
  with $H(L) = \log(F(e^L)/e^L)$ as in the proof of \Cref{prop:sync_v_convergence}, $\psi' := \psi(d/2) + \log 2 = \Exp{\log X}$, and $\sigma^2 := \Var(\log X) = \psi_1(d/2) < \infty$ (the trigamma function is finite for every $d \ge 1$). Set $\mu(L) := H(L) + \psi' = \Exp{L_{st}^{(i+1)} - L_{st}^{(i)} \cond L_{st}^{(i)} = L}$ and define the scalar per-pair quadratic drift
  \begin{eqnarray*}
    \Delta(L) \;:=\; \Exp{(L_{st}^{(i+1)})^2 - (L_{st}^{(i)})^2 \cond L_{st}^{(i)} = L}.
  \end{eqnarray*}
  Expanding $(L + Y)^2 - L^2 = 2LY + Y^2$ with $Y := L_{st}^{(i+1)} - L_{st}^{(i)}$ and completing the square via $\mu + L = \log F(e^L) + \psi'$,
  \begin{eqnarray}
    \label{eq:per_pair_quad_drift}
    \Delta(L) \;=\; \mu(L)^2 + \sigma^2 + 2L\mu(L) \;=\; \Brack{\log F(e^L) + \psi'}^2 - L^2 + \sigma^2.
  \end{eqnarray}
  The function $\Delta$ is continuous on $\RR$ and $\Delta(L) \to -\infty$ as $|L| \to \infty$:
  \begin{itemize}
    \item[] \emph{Left tail.} As $L \to -\infty$, $F(u) = u/r^2 + O(u^2)$ gives $\log F(e^L) = L - 2\log r + o(1)$, so $\mu(L) \to \psi' - 2\log r = \rho_d > 0$ by subcriticality $r < r_c(d)$. Hence $\mu(L)^2$ and $\sigma^2$ stay bounded while $2L\mu(L) \sim 2L\rho_d \to -\infty$.
    \item[] \emph{Right tail.} As $L \to +\infty$, $F(e^L) \to 2$, so $\log F(e^L) + \psi'$ is bounded and $\Delta(L) = O(1) - L^2 \to -\infty$.
  \end{itemize}
  In particular $B := \sup_{L \in \RR} \Delta(L) < \infty$ (depending only on $r, d$), and we can choose $M > 0$ so that $\Delta(L) \le -B\binom{n}{2} - 1$ whenever $|L| \ge M$.

  The distance vector itself does not give a chain we can analyse through a Lebesgue density: for $n \ge d+2$ the vector $\bu^{(i)} = \bu(\bZ_i)$ is confined to the cone of squared Euclidean distance matrices of $n$ points in $\RR^d$, a submanifold of dimension $nd - \binom{d+1}{2} < \binom{n}{2}$, so its law is singular with respect to Lebesgue measure on $(0,\infty)^{\binom{n}{2}}$ (\Cref{rem:edm_state_space}). We therefore run the Foster--Lyapunov argument on the position chain $\Set{\bZ_i}$ on $(\RR^d)^n$ itself, whose Gaussian transition kernel \emph{does} admit a strictly positive density, and recover $\pi^{\bu}$ as the image of $\pi_{\bZ}$ under the distance map.

  Define the Lyapunov function on $(\RR^d)^n$ through the pairwise distances,
  \begin{eqnarray*}
    V(\bz) \;:=\; \sum_{s<t} \bigl(\log \norm{z_s - z_t}^2\bigr)^2 \;\ge\; 0,
  \end{eqnarray*}
  so that $V(\bZ_i) = \sum_{s<t} (L_{st}^{(i)})^2$. By \emph{I.1}, the conditional law of each $u_{st}^{(i+1)}$ given $\bZ_i$ depends on $\bZ_i$ only through $u_{st}^{(i)}$, so the drift of $V$ decomposes as a sum of scalar per-pair drifts:
  \begin{eqnarray}
    \label{eq:sum_quad_drift}
    \Exp{V(\bZ_{i+1}) \cond \bZ_i} - V(\bZ_i) \;=\; \sum_{s<t} \Delta(L_{st}^{(i)}).
  \end{eqnarray}

  \emph{I.3: Small set and drift off it.} Define $C := \Set{\bz \in (\RR^d)^n : \max_{s<t} \Abs{\log \norm{z_s - z_t}^2} \le M}$. For $\bz \notin C$ there exists at least one pair $(s_0, t_0)$ with $\Abs{L_{s_0 t_0}} > M$; that pair contributes $\Delta(L_{s_0 t_0}) \le -B\binom{n}{2} - 1$ by the choice of $M$, while every other pair contributes at most $B$, giving
  \begin{eqnarray*}
    \sum_{s<t} \Delta(L_{st}) \;\le\; -B\binom{n}{2} - 1 + B\Brack{\binom{n}{2} - 1} \;=\; -B - 1 \;\le\; -1.
  \end{eqnarray*}
  For $\bz \in C$, the trivial bound $\sum_{s<t} \Delta(L_{st}) \le B \binom{n}{2}$ holds. Combining,
  \begin{eqnarray*}
    \Exp{V(\bZ_{i+1}) \cond \bZ_i = \bz} \;\le\; V(\bz) - 1 + \Brack{B\binom{n}{2} + 1}\,\mathbf{1}_C(\bz),
  \end{eqnarray*}
  which is the Foster--Lyapunov drift criterion (Theorem~11.3.4 of \cite{meyn1993markov}) for $\Set{\bZ_i}$ on $(\RR^d)^n$.

  \emph{I.4: Irreducibility, small-set minorisation, aperiodicity.} Given $\bZ_i = \bz$ with pairwise-distinct coordinates, the law of $\bZ_{i+1}$ is $\mcN(0, R(\bu(\bz)) \otimes I_d)$, which has a strictly positive continuous density on all of $(\RR^d)^n$: the correlation matrix $R(\bu(\bz))$ is positive definite whenever the points are distinct ($R_{ss} = 1$, $|R_{st}| = \phi(\norm{z_s - z_t}) < 1$, and positive-definiteness of the RBF correlation kernel on any finite set of distinct points is standard; see Proposition~2 of \cite{dunlop2018how}), so $R(\bu(\bz)) \otimes I_d$ is nonsingular. Since this density has full support for every such $\bz$, the chain $\Set{\bZ_i}$ is $\psi$-irreducible with respect to Lebesgue measure on $(\RR^d)^n$. On $C$ the squared distances lie in the compact box $[e^{-M}, e^{M}]^{\binom{n}{2}}$, and we claim that the family $\mcR_C := \Set{R(\bu(\bz)) : \bz \in C}$ is a compact set of positive-definite matrices, so that the eigenvalues of $R(\bu(\bz))$ are bounded away from $0$ and $\infty$ uniformly over $\bz \in C$ (the smallest eigenvalue being a continuous strictly positive function, and the largest a continuous finite one, on a compact set of positive-definite matrices). To see the claim, note first that the image $\bu(C)$ is compact: it is bounded, being contained in the box, and it is closed, since if $\bu(\bz_k) \to \bu_*$ with $\bz_k \in C$ then, after translating each configuration so that its first point is at the origin (translations change no distance, hence preserve membership in $C$), all $\bz_k$ lie in a fixed compact subset of $(\RR^d)^n$, and a convergent subsequence $\bz_{k_j} \to \bz_*$ realises $\bu_* = \bu(\bz_*) \in \bu(C)$. Hence $\mcR_C = R(\bu(C))$ is compact, $R$ being continuous. Second, every matrix in $\mcR_C$ is positive definite, since on $C$ the distances are bounded below by $e^{-M/2} > 0$, so the underlying points are distinct and the strict positive-definiteness of the RBF kernel applies as above. We remark that the strict positive-definiteness of each individual matrix is essential here, and boundedness of the entries alone would not suffice: singular correlation matrices with entries bounded away from $0$ and $1$ exist (e.g.\ the Gram matrix of three coplanar unit vectors at pairwise angles of $60^\circ$). The corresponding centred Gaussian densities are therefore uniformly bounded below on any fixed compact ball $B_0 \subset (\RR^d)^n$, yielding the minorisation $P(\bz, \cdot) \ge \eta\, \nu(\cdot)$ for all $\bz \in C$, with $\eta > 0$ and $\nu$ the normalised restriction of Lebesgue measure to $B_0$. Hence $C$ is a small set with $m = 1$, and the chain is strongly aperiodic. (Note $C$ is the small set even though it is unbounded under translations: the kernel $\mcN(0, R(\bu(\bz)) \otimes I_d)$ is centred at the origin and depends on $\bz$ only through the distances, which are bounded on $C$.)

  \emph{I.5: Ergodic theorem.} Foster's drift criterion (Theorem~11.3.4 of \cite{meyn1993markov}) with the drift from \emph{I.3} and the small set from \emph{I.4} yields positive Harris recurrence of $\Set{\bZ_i}$ and a unique invariant probability measure $\pi_{\bZ}$ on $(\RR^d)^n$. The Aperiodic Ergodic Theorem (Theorem~13.0.1 of \cite{meyn1993markov}) then gives $\norm{\Law(\bZ_i) - \pi_{\bZ}}_{\mathrm{TV}} \to 0$ from every initial condition $\bZ_1$ with pairwise-distinct coordinates. Let $\pi^{\bu}$ be the image of $\pi_{\bZ}$ under the distance map $\bz \mapsto \bu(\bz)$; it is invariant for the distance chain $\Set{\bu^{(i)}}$ (well defined and Markov by \emph{I.1}), and $\norm{\Law(\bu^{(i)}) - \pi^{\bu}}_{\mathrm{TV}} \to 0$ as well, the distance map being a contraction in total variation. The marginal of $\pi^{\bu}$ at each pair $(s,t)$ is stationary for the scalar chain (by \emph{I.1}), hence equal to $\pi^{\bv}$ by uniqueness in \Cref{thm:sync_v_convergence_d}~(ii). Uniqueness of $\pi^{\bu}$ follows from that of $\pi_{\bZ}$: any invariant law of $\Set{\bu^{(i)}}$ lifts through $k(\bu, \cdot) = \mcN(0, R(\bu) \otimes I_d)$ to an invariant law of $\Set{\bZ_i}$, necessarily $\pi_{\bZ}$, whose image is $\pi^{\bu}$.

  \emph{Part II: explicit TV bound and mixture representation.}

  Denote $P(\bz, A) := \Prob{\bZ_{i+1} \in A \cond \bZ_i = \bz} = \mcN(0, R(\bz) \otimes I_d)(A)$. This depends on $\bz$ only through $\bu(\bz) = (\norm{z_s - z_t}^2)_{s<t}$, so there is a probability kernel $k(\bu, \cdot) := \mcN(0, R(\bu) \otimes I_d)$ on $(\RR^d)^n$ with $P(\bz, A) = k(\bu(\bz), A)$. By the Markov property,
  \begin{eqnarray*}
    \Prob{\bZ_{i+1} \in A} \;=\; \Exp{k(\bu^{(i)}, A)} \;=\; \int k(\bu, A)\, d\Law(\bu^{(i)})(\bu),
  \end{eqnarray*}
  and by \eqref{eq:sync_Z_limit_measure}, $\pi_{\bZ}(A) = \int k(\bu, A)\, d\pi^{\bu}(\bu)$. Subtracting and using $k(\bu, A) \in [0, 1]$,
  \begin{eqnarray*}
    \bigl|\Prob{\bZ_{i+1} \in A} - \pi_{\bZ}(A)\bigr| \;=\; \Big|\!\int k(\bu, A)\, d(\Law(\bu^{(i)}) - \pi^{\bu})(\bu)\Big| \;\le\; \norm{\Law(\bu^{(i)}) - \pi^{\bu}}_{\mathrm{TV}}.
  \end{eqnarray*}
  Taking the supremum over $A$ and shifting $i \mapsto i-1$ gives \eqref{eq:sync_Z_tv_bound}. By Part~I, the RHS tends to $0$ in TV.

  \emph{Mixture representation.} Existence and uniqueness of $\pi_{\bZ}$ were established in Part~I; it remains only to identify it with \eqref{eq:sync_Z_limit_measure}. Invariance of $\pi_{\bZ}$ gives $\pi_{\bZ}(A) = \int P(\bz, A)\, d\pi_{\bZ}(\bz) = \int k(\bu, A)\, d\pi^{\bu}(\bu)$, with $\pi^{\bu}$ the image of $\pi_{\bZ}$ under $\bz \mapsto \bu(\bz)$; this is exactly \eqref{eq:sync_Z_limit_measure}.
\end{proof}

\begin{proof}[Proof of \Cref{prop:sync_not_gaussian}]
  \emph{Marginals.} For every $i \ge 1$ and $t \in \Set{1, 2}$, conditional on $\bZ_{i-1}$ each coordinate of $\bZ_{i,t} = G_{i-1}(Z_{i-1,t})$ is a marginal evaluation of a centred unit-variance GP, so $\bZ_{i,t} \cond \bZ_{i-1} \sim \mcN(0, \phi(0)\, I_d) = \mcN(0, I_d)$. The $(t,t)$-block is independent of $\bZ_{i-1}$, so unconditionally $\bZ_{i,t} \sim \mcN(0, I_d)$. Weak convergence then forces $\bZ_{\infty, t} \sim \mcN(0, I_d)$. It remains to prove non-joint-Gaussianity.

  \emph{Step 1: any Gaussian joint limit has isotropic $V_\infty$.} Let $V_i := \bZ_{i,1} - \bZ_{i,2} \in \RR^d$. Conditional on $\bZ_{i-1}$, the $d$ output coordinates of $G_{i-1}$ are i.i.d., so the $d$ coordinates of $V_i = G_{i-1}(\bZ_{i-1,1}) - G_{i-1}(\bZ_{i-1,2})$ are i.i.d.\ with marginal $\mcN(0, 2(1 - \phi(\norm{V_{i-1}})))$. Hence
  \begin{eqnarray*}
    V_i \cond \bZ_{i-1} \;\sim\; \mcN\!\bigl(0,\, 2(1 - \phi(\norm{V_{i-1}}))\, I_d\bigr),
  \end{eqnarray*}
  a scalar multiple of $I_d$ depending on $V_{i-1}$ only through $\norm{V_{i-1}}$. Therefore the unconditional law of $V_i$ is a mixture of isotropic Gaussians and is itself rotationally invariant in $\RR^d$ for every $i \ge 1$; so is its weak limit $V_\infty$. Suppose now, for contradiction, that $(\bZ_{\infty,1}, \bZ_{\infty,2})$ is jointly Gaussian. Then $V_\infty$ is a centred Gaussian in $\RR^d$; combined with rotational invariance, its covariance matrix is a scalar multiple of $I_d$, i.e.\ $V_\infty \sim \mcN(0, \sigma_v^2 I_d)$ for some $\sigma_v^2 \in [0, \infty)$. The case $\sigma_v^2 = 0$ gives $V_\infty = 0$ a.s., which contradicts \Cref{thm:sync_v_convergence_d}~(ii) (nontrivial stationary law on $(0, \infty)$), so $\sigma_v^2 > 0$. Therefore $u_\infty := \norm{V_\infty}^2 \stackrel d= \sigma_v^2\, \chi^2_d$, equivalently $u_\infty \stackrel d= e^\zeta \chi^2_d$ with $\zeta := \log \sigma_v^2 \in \RR$.

  \emph{Step 2: no $e^\zeta \chi^2_d$ is stationary for \eqref{eq:kernel_corr_converegnce_d}.} Suppose $u \stackrel d= e^\zeta X$ with $X \sim \chi^2_d$, and let $X_i \sim \chi^2_d$ be the independent innovation in \eqref{eq:kernel_corr_converegnce_d}, so that the one-step image is $u' = 2(1 - \phi(\sqrt u))\, X_i = A(X)\, X_i$ with
  \begin{eqnarray*}
    A(X) \;:=\; 2\bigl(1 - \phi(e^{\zeta/2}\sqrt X)\bigr) \;=\; 2\bigl(1 - e^{-e^\zeta X / (2 r^2)}\bigr).
  \end{eqnarray*}
  Stationarity requires $u' \stackrel d= e^\zeta X''$ with $X'' \sim \chi^2_d$, i.e.\ $A(X)\, X_i \stackrel d= e^\zeta X''$. Taking logarithms,
  \begin{eqnarray}
    \label{eq:sync_gaussian_conv_id}
    \log A(X) + \log X_i \;\stackrel d=\; \zeta + \log X'',
  \end{eqnarray}
  where $\log X_i$ and $\log X''$ are both distributed as $\log \chi^2_d$ and each is independent of the other terms on its side.

  Recall that the \emph{characteristic function} of a random variable $Y$ is $\varphi_Y(t) := \Exp{e^{i t Y}}$ for $t \in \RR$; $\varphi_Y$ is continuous with $\varphi_Y(0) = 1$, is multiplicative under convolution (so $\varphi_{Y_1 + Y_2} = \varphi_{Y_1}\, \varphi_{Y_2}$ for independent $Y_1, Y_2$), and uniquely determines the law of $Y$ (\cite{durrett2019probability}). Applying $\varphi$ to both sides of \eqref{eq:sync_gaussian_conv_id} and using independence yields, for every $t \in \RR$,
  \begin{eqnarray}
    \label{eq:sync_gaussian_cf_id}
    \varphi_{\log A(X)}(t)\, \varphi_{\log \chi^2_d}(t) \;=\; e^{i t \zeta}\, \varphi_{\log \chi^2_d}(t).
  \end{eqnarray}
  We claim $\varphi_{\log \chi^2_d}$ is non-vanishing on all of $\RR$. By the Mellin transform of the $\chi^2_d$ density,
  \begin{eqnarray*}
    \varphi_{\log \chi^2_d}(t) \;=\; \Exp{X^{i t}} \;=\; \frac{2^{i t}\, \Gamma(d/2 + i t)}{\Gamma(d/2)},
  \end{eqnarray*}
  and $\Gamma$ has no zeros on $\mathbb{C}$ (\cite{abramowitz1964handbook}), so $\Gamma(d/2 + i t) \neq 0$ for all $t \in \RR$ and the claim follows. Dividing \eqref{eq:sync_gaussian_cf_id} through by $\varphi_{\log \chi^2_d}(t)$ gives
  \begin{eqnarray*}
    \varphi_{\log A(X)}(t) \;=\; e^{i t \zeta} \qquad \text{for every } t \in \RR.
  \end{eqnarray*}
  The right-hand side is the characteristic function of the point mass $\delta_{\zeta}$, so by uniqueness of characteristic functions $\log A(X) \stackrel{d}{=} \delta_\zeta$, i.e.\ $A(X)$ is almost surely equal to $e^\zeta$. But $X \mapsto A(X) = 2(1 - e^{-e^\zeta X / (2 r^2)})$ is strictly increasing and non-constant on $(0, \infty)$, and $X \sim \chi^2_d$ has full support on $(0, \infty)$, so $A(X)$ is non-degenerate. Contradiction.

  \emph{Step 3: combining.} By Step~2, no $e^\zeta \chi^2_d$ is invariant under \eqref{eq:kernel_corr_converegnce_d}. Since the stationary law of $u_i$ is unique by \Cref{thm:sync_v_convergence_d}~(ii) and the law of $u_i = \norm{V_i}^2$ converges to this stationary law, $u_\infty$ is not of the form $e^\zeta \chi^2_d$ for any $\zeta \in \RR$. This contradicts the Gaussianity-plus-isotropy consequence of Step~1. Hence $(\bZ_{\infty,1}, \bZ_{\infty,2})$ is not jointly Gaussian.
\end{proof}

\begin{proof}[Proof of \Cref{cor:sync_log_stationary}]
  Finiteness of $\Expsubidx{\pi^{\bv}}{|L|}$: the Lyapunov function $V(L) = (L_0 - L)^+ + (L - L^0)^+$ from the proof of \Cref{prop:sync_v_convergence}~(ii) satisfies $\Expsubidx{\pi^{\bv}}{V} < \infty$ (standard for positive Harris recurrent chains with a Foster--Lyapunov drift function; Theorem~14.3.7 of \cite{meyn1993markov}), and $|L| \le V(L) + \max(|L_0|, |L^0|)$, so $\Expsubidx{\pi^{\bv}}{|L|} < \infty$.

  (a) We first verify that the expectations below are finite: as $L \to -\infty$, $H(L) = \log F(e^L) - L \to L - 2\log r - L + o(1) = -2\log r$ stays bounded, and as $L \to +\infty$, $\log F(e^L) \to \log 2$ stays bounded so $H(L) = \log 2 - L + o(1)$, giving the linear envelope $|H(L)| \le |L| + C$ for a constant $C = C(r)$. Hence $\Expsubidx{\pi^{\bv}}{|H(L)|} \le \Expsubidx{\pi^{\bv}}{|L|} + C < \infty$ since $\Expsubidx{\pi^{\bv}}{|L|} < \infty$ from the previous paragraph.
  Under $\pi^{\bv}$, $L_i \stackrel{d}{=} L_{i+1}$, so taking expectations in \eqref{eq:L_chain} (with $\log g_i^2$ replaced by $\log X_i$, $X_i \sim \chi^2_d$) and using $X_i \perp L_i$ gives $0 = \Expsubidx{\pi^{\bv}}{H(L)} + \Exp{\log X} = \Expsubidx{\pi^{\bv}}{H(L)} + \psi(d/2) + \log 2$.

  (b) We claim $H$ is strictly concave and strictly decreasing on $\RR$. With $x := e^L/(2 r^2)$,
  \begin{eqnarray*}
    H'(L) \;=\; \frac{x}{e^x - 1} - 1,
  \end{eqnarray*}
  and $d/dx[x/(e^x - 1)] = [e^x(1-x) - 1]/(e^x - 1)^2 < 0$ for $x > 0$ (the numerator is $0$ at $x = 0$ and has derivative $-x e^x < 0$). Since $dx/dL = x > 0$, $H'$ is strictly decreasing in $L$ and negative (as $H'(L) \to 0^-$ at $L \to -\infty$); hence $H$ is strictly concave and strictly decreasing. Jensen's inequality, combined with (a), gives
  \begin{eqnarray*}
    H(\Expsubidx{\pi^{\bv}}{L}) \;\ge\; \Expsubidx{\pi^{\bv}}{H(L)} \;=\; -\psi(d/2) - \log 2 \;=\; H(L_*),
  \end{eqnarray*}
  and monotonicity of $H$ inverts this to $\Expsubidx{\pi^{\bv}}{L} \le L_*$. The explicit form \eqref{eq:sync_L_star} follows from $H(L_*) = -\psi(d/2) - \log 2$, which upon exponentiating reads $F(e^{L_*})/e^{L_*} = 1/r_c(d)^2$: setting $\alpha := e^{L_*}/(2 r^2)$ yields $(1 - e^{-\alpha})/\alpha = (r/r_c(d))^2$. The map $\alpha \mapsto (1 - e^{-\alpha})/\alpha$ is smooth and strictly decreasing from $1$ to $0$ on $(0, \infty)$, so $\alpha$ is uniquely determined by $r < r_c(d)$. The limits as $r \downarrow 0$ (RHS $\to 0$, $\alpha \to \infty$, $(1-e^{-\alpha})/\alpha \sim 1/\alpha$, so $\alpha \sim (r_c/r)^2$ and $2 r^2 \alpha \to 2 r_c^2$) and $r \uparrow r_c(d)$ (RHS $\to 1$, $\alpha \to 0$, $L_* \to -\infty$) are direct.
\end{proof}

\begin{proof}[Proof of \Cref{cor:sync_subcrit_d_scaling}]
  $L_* = \log(2 r^2 \alpha)$ with $\alpha$ determined by $(1 - e^{-\alpha})/\alpha = (r/r_c(d))^2$; at $r = \lambda\,r_c(d)$ this RHS equals $\lambda^2$, so $\alpha = \alpha(\lambda)$ is independent of $d$. Then $L_* = \log(2\,\lambda^2 r_c(d)^2 \alpha(\lambda)) = \log(\lambda^2 \alpha(\lambda)) + \log 2 + \log r_c(d)^2 = \log(\lambda^2 \alpha(\lambda)) + \log 2 + (\psi(d/2) + \log 2)$.
\end{proof}

\section{Additional empirical figures}
\label{sec:appendix_more_figs}

This appendix collects the empirical figures supporting \Cref{sec:experiments} that are not in the main paper. All simulations are CPU-only Python (NumPy / SciPy / matplotlib / scikit-learn); the full set of experiments finishes well within an hour on a single laptop core.

\Cref{fig:appendix_trajectories_above} verifies the supercritical decay rate of \Cref{thm:sync_v_convergence_d}~(i) across $d \in \Set{1, 10, 100}$. \Cref{fig:appendix_trajectories_below,fig:appendix_stationary_multidim} are the $d = 1, 10$ companions of \Cref{fig:exp_d100_main}: subcritical trajectories of $\log v_i^2$ across all three dimensions, and the empirical stationary law of $\log v_i^2$ across all three dimensions. At $d = 1$ near-critical $\lambda$ the convergence becomes very slow and the stationary law has an extremely heavy left tail. \Cref{fig:appendix_trajectories_std} reports the sample std of $\log v_i^2$ across the $300$ i.i.d.\ replicates used to compute the mean trajectories in \Cref{fig:appendix_trajectories_above,fig:appendix_trajectories_below}; dividing this std by $\sqrt{300}$ yields the standard error of the corresponding mean trajectory estimator.

The remaining figures are per-chain visualisations of $\pi_{\bZ}$, all using $n = 1000$ and five i.i.d.\ chains per $\lambda$, with $\lambda$ values chosen per dimension to span the regime transition (see ``From bounds to a choice of $\lambda$'' in \Cref{sec:experiments}): $d = 1$ at $\lambda \in \Set{0.10, 0.30, 0.60, 0.85}$ and depth $i = 300$, $d = 10$ at $\lambda \in \Set{0.66, 0.91, 0.95, 0.97}$ and depth $i = 600$, $d = 100$ at $\lambda \in \Set{0.90, 0.99, 0.995}$ and depth $i = 1000$. The smallest $\lambda$ in each panel is the near-independent Gaussian baseline against which the structure at the larger $\lambda$ should be read.

\Cref{fig:appendix_pi_Z_pca_d100} is the linear-PCA companion to the main-paper t-SNE figure at $d = 100$: PC1--PC2 of the same chains as \Cref{fig:exp_pi_Z_tsne_main}, with no non-linear embedding step. The chain-specific bananas, U-shapes, and arcs that appear at $\lambda = 0.99, 0.995$ in the t-SNE figure are also visible in PCA, confirming that the structure is in the data and not invented by t-SNE. See also \Cref{fig:appendix_pca_density_d100}, where we show the \emph{density} rather than just the scatter of the PCA plots, which makes the multimodality even clearer. \Cref{fig:appendix_pi_Z_hist_d1} shows the $d = 1$ case as a histogram of the $n$ scalar entries of $\bZ_{i, t}$ per chain (the top PC of $1$-D data is the data itself). \Cref{fig:appendix_pi_Z_pca_d10,fig:appendix_pi_Z_tsne_d10} show per-chain PCA and t-SNE at $d = 10$, where the visible-structure window is wider than at $d = 100$ and the transition Gaussian-cloud $\to$ clusters $\to$ ribbons spans four $\lambda$ rows. \Cref{fig:appendix_pca_density_d10} shows the empirical density of the PCA projections at $d = 10$.

\begin{figure}[t]
  \centering
  \includegraphics[width=0.95\textwidth]{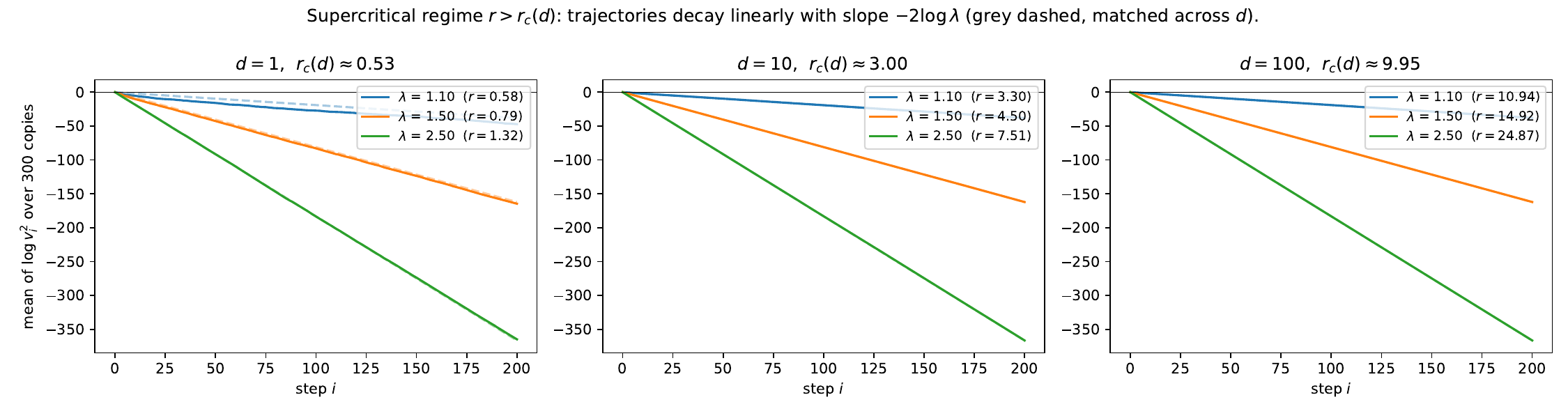}
  \caption{Supercritical regime $r > r_c(d)$: trajectories of $\log v_i^2$ for $\lambda \in \Set{1.1, 1.5, 2.5}$ across $d = 1, 10, 100$. Each curve is the mean of $\log v_i^2$ over $300$ i.i.d.\ trajectories starting at $v_1 = 1$; the grey dashed reference line has slope $-2\log\lambda$ predicted by \Cref{thm:sync_v_convergence_d}~(i), and is hidden behind the simulated curve in every case. At matched $\lambda$, the decay rate is the same across dimensions, confirming that the rate bound from \Cref{thm:sync_v_convergence_d} is $d$-independent in units of $\lambda$. The $d = 100$ pane is the same as the left pane of \Cref{fig:exp_d100_main}, repeated here for direct cross-$d$ comparison.}
  \label{fig:appendix_trajectories_above}
\end{figure}

\begin{figure}[t]
  \centering
  \includegraphics[width=0.95\textwidth]{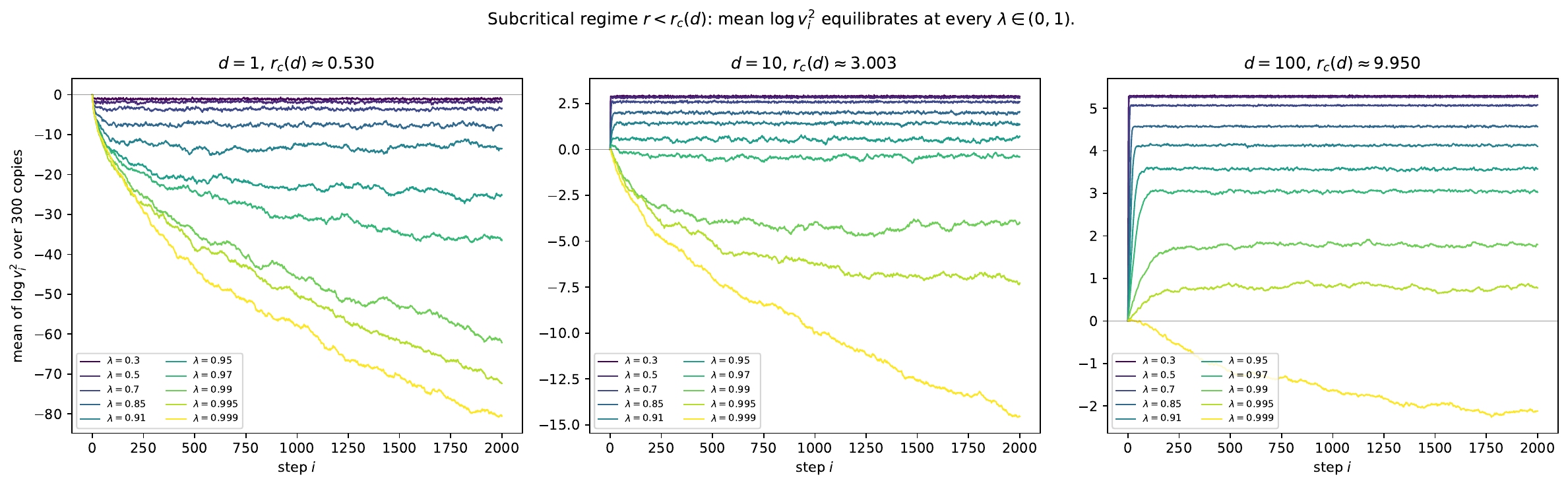}
  \caption{Subcritical regime $r < r_c(d)$: trajectories of $\log v_i^2$ for ten $\lambda$ values spanning $0.30$ to $0.999$, across $d = 1, 10, 100$, all panels at uniform depth $2000$. Each curve is the mean of $\log v_i^2$ over $300$ i.i.d.\ trajectories starting at $v_1 = 1$, colour-coded from low (purple) to high (yellow). The convergence to equilibrium claimed by \Cref{thm:sync_v_convergence_d}~(ii) is observed; it slows as $\lambda \uparrow 1$, especially at $d = 1$, where the empirical mean of $\log v^2$ continues drifting beyond the depth shown.}
  \label{fig:appendix_trajectories_below}
\end{figure}

\begin{figure}[t]
  \centering
  \includegraphics[width=0.95\textwidth]{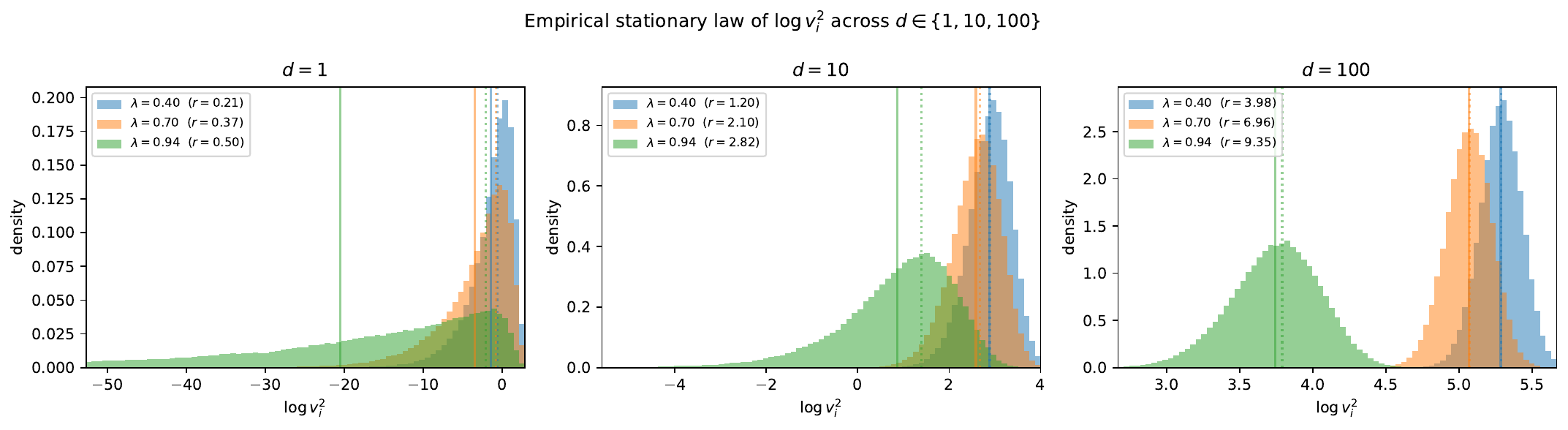}
  \caption{Empirical stationary law of $\log v_i^2$ at $\lambda \in \Set{0.40, 0.70, 0.94}$, overlaid per dimension ($d = 1, 10, 100$ left to right). Each histogram is from a chain of $2 \cdot 10^5$ samples after a $2 \cdot 10^4$-step burn-in; the $d = 1$ pane trims the leftmost $8\%$ of each sample to keep the bulk visible. Solid verticals: empirical stationary mean $\Expsubidx{\pi^{\bv}}{\log v^2}$ (from the full, unclipped sample). Dotted verticals: Jensen upper bound $L_*(r,d)$ of \Cref{cor:sync_log_stationary}~(b). The bound is nearly tight at $d = 100$ and loose at $d = 1, \lambda = 0.94$ (the heavy left tail pulls the empirical mean far below $L_*$). The $d = 100$ histograms are markedly narrower than at $d = 10$ or $d = 1$ because $\log X$ for $X \sim \chi^2_d$ concentrates at rate $d^{-1/2}$.}
  \label{fig:appendix_stationary_multidim}
\end{figure}

\begin{figure}[t]
  \centering
  \includegraphics[width=0.95\textwidth]{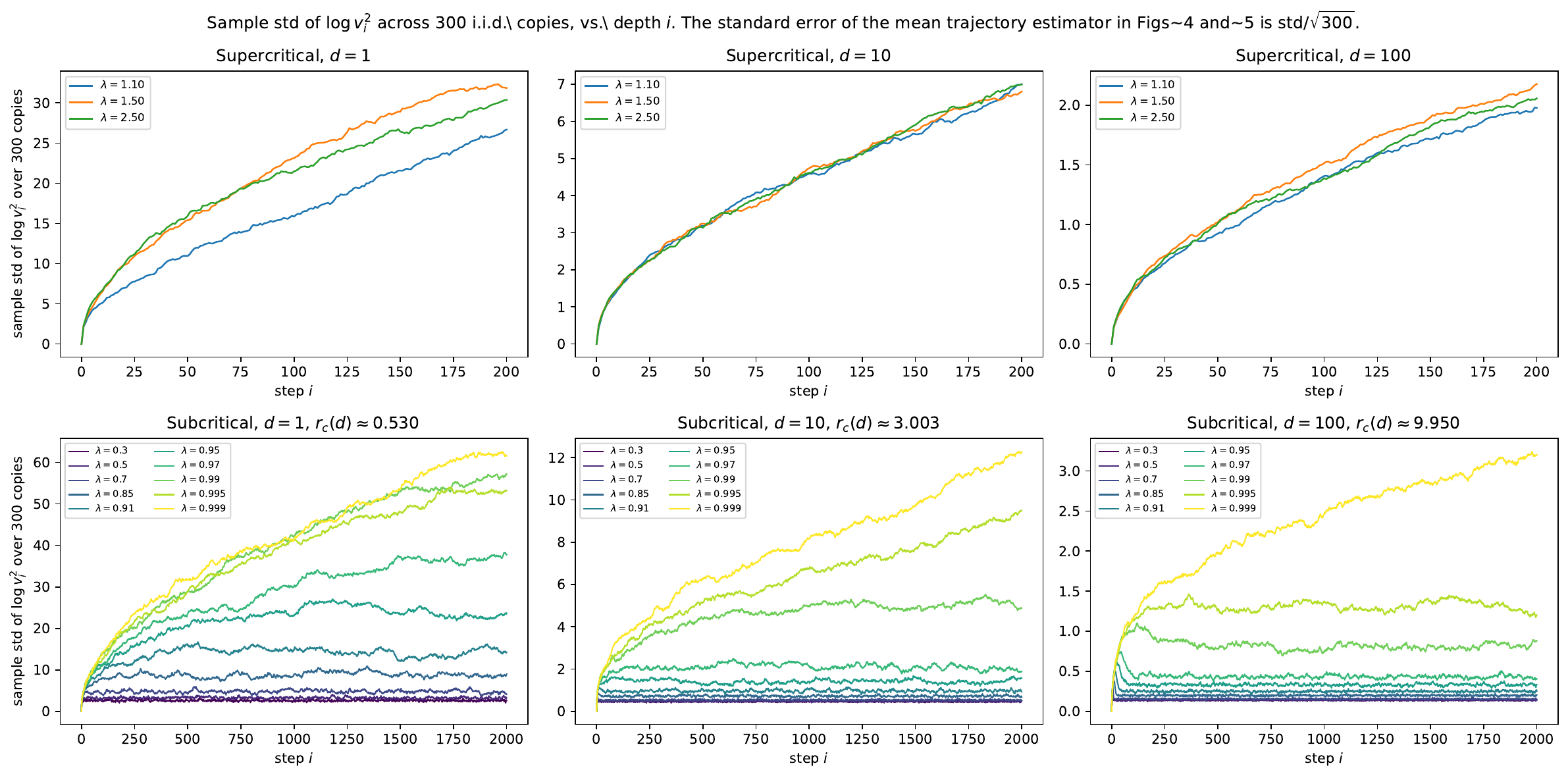}
  \caption{Sample std of $\log v_i^2$ across the $300$ i.i.d.\ copies used in \Cref{fig:appendix_trajectories_above,fig:appendix_trajectories_below}, vs.\ depth $i$. \emph{Top row:} supercritical regime, $\lambda \in \Set{1.1, 1.5, 2.5}$, depth $200$. \emph{Bottom row:} subcritical regime, $\lambda \in \Set{0.30, \ldots, 0.999}$, depth $2000$; the std saturates as the chain approaches its stationary law, with saturation level increasing in $\lambda$. The standard error of the mean trajectory estimator displayed in \Cref{fig:appendix_trajectories_above,fig:appendix_trajectories_below} is the std shown here divided by $\sqrt{300}$; this is below figure resolution at $d = 10, 100$ for every $\lambda$ used in the paper, and only the most near-critical $\lambda \in \Set{0.99, 0.995, 0.999}$ at $d = 1$ pushes the standard error to a fraction of the mean trajectory's range, so the empirical mean of $\log v^2$ at $d = 1$ near-critical $\lambda$ is an indicative estimate rather than a tight one.}
  \label{fig:appendix_trajectories_std}
\end{figure}

\begin{figure}[t]
  \centering
  \includegraphics[width=\textwidth]{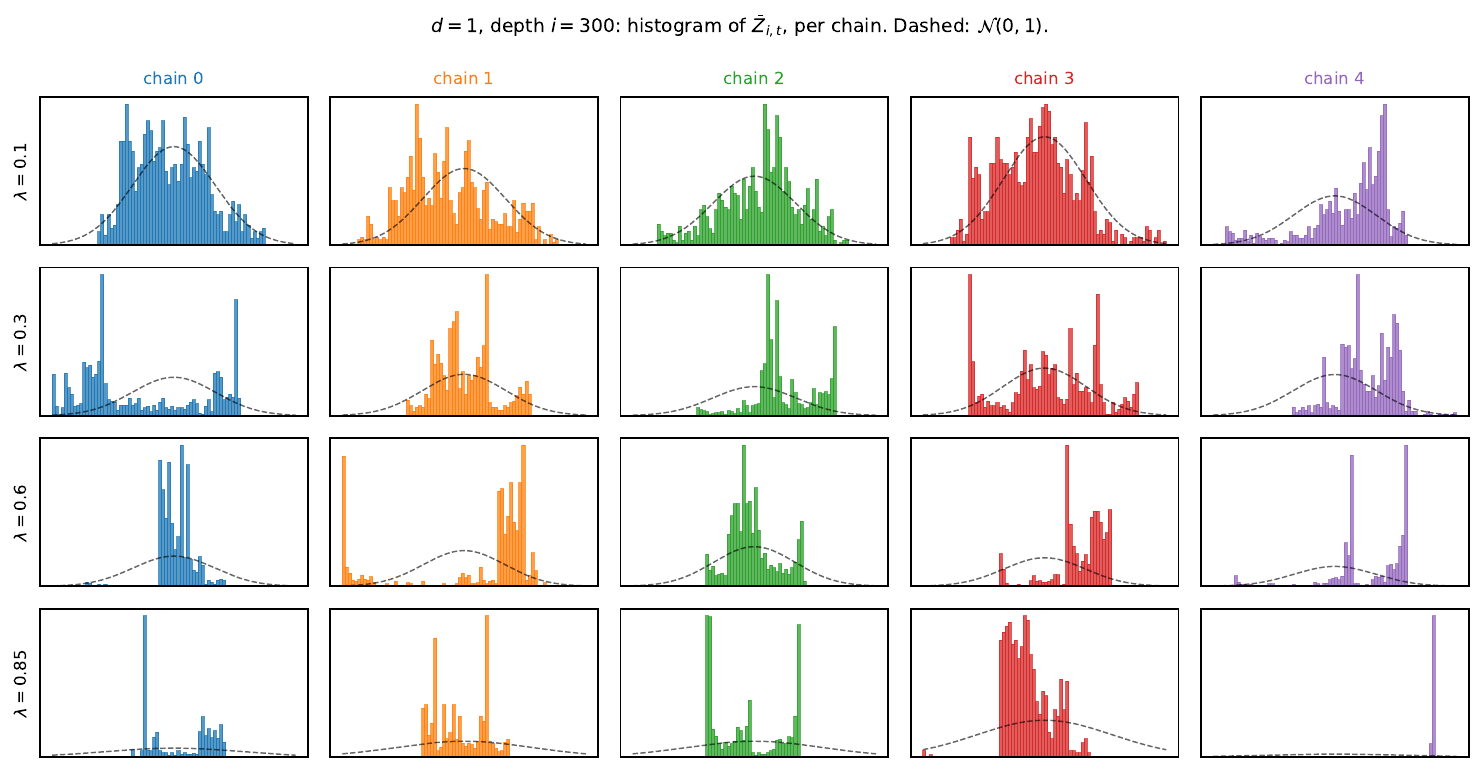}
  \caption{Per-chain histograms of $\bZ_{i,t} \in \RR$ at depth $i = 300$, $d = 1$. Rows: $\lambda \in \Set{0.10, 0.30, 0.60, 0.85}$. Dashed: $\mathcal N(0, 1)$. The marginal stays exactly $\mathcal N(0,1)$ at every depth, so any departure from the dashed curve is a within-chain dependence effect. The top row ($\lambda = 0.10$) is essentially indistinguishable from a $\mathcal N(0,1)$ sample; as $\lambda \uparrow 1$ the histograms become heavily peaked / asymmetric and vary visibly between chains.}
  \label{fig:appendix_pi_Z_hist_d1}
\end{figure}

\begin{figure}[t]
  \centering
  \includegraphics[width=\textwidth]{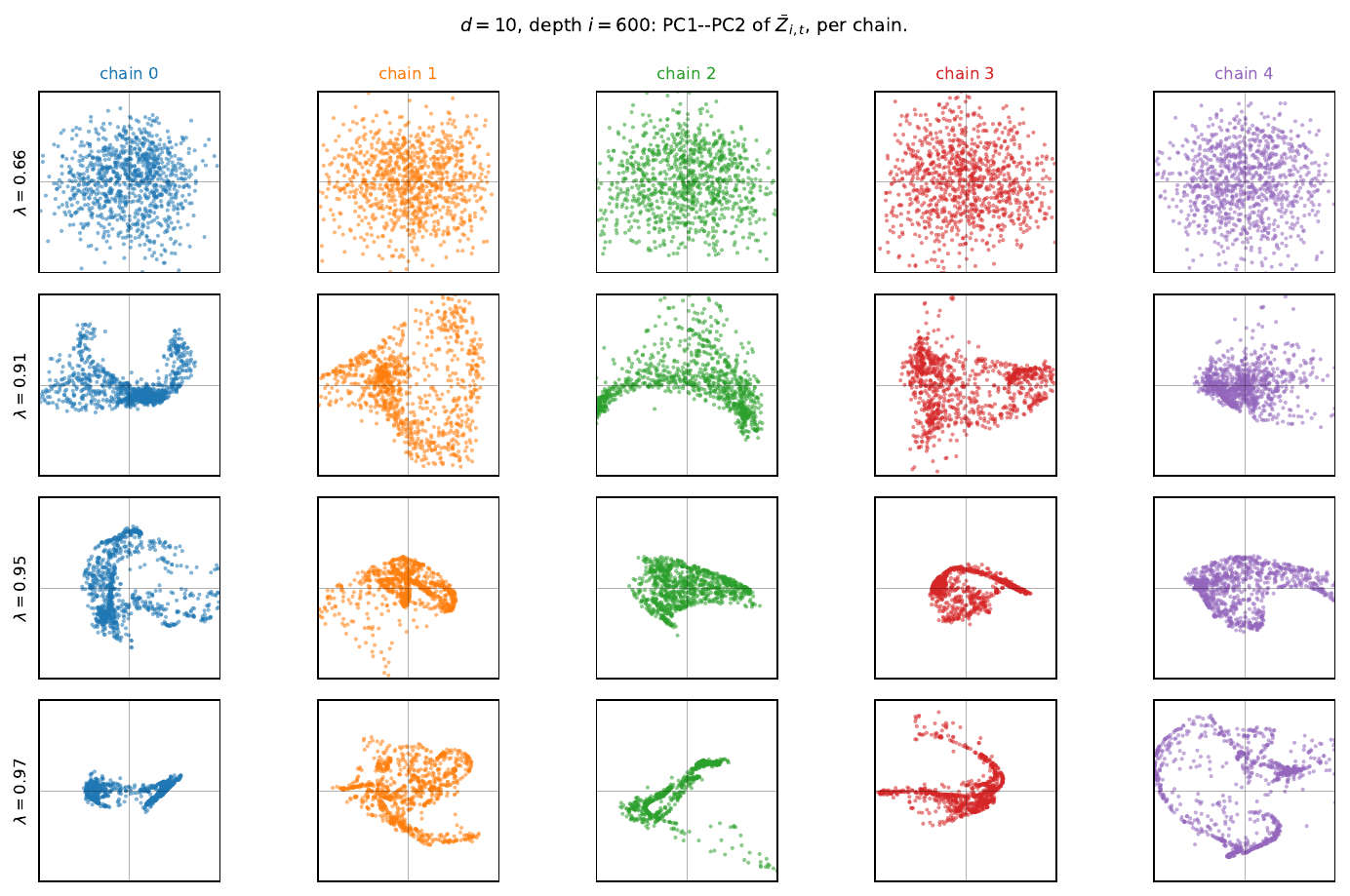}
  \caption{Per-chain PC1--PC2 scatters of $\bZ_{i,t} \in \RR^{10}$ at depth $i = 600$, five i.i.d.\ chains (columns); $n = 1000$ points per panel, centred and projected onto the top two PCs. Rows: $\lambda \in \Set{0.66, 0.91, 0.95, 0.97}$. The top two rows are still close to a Gaussian cloud; the lower rows show the cluster-then-ribbon transition characteristic of the wider visible-structure window at moderate $d$. See also \Cref{fig:appendix_pca_density_d10}, where we plot the \emph{density} of the points above, which makes the structure and multimodality even clearer.}
  \label{fig:appendix_pi_Z_pca_d10}
\end{figure}

\begin{figure}[t]
  \centering
  \includegraphics[width=\textwidth]{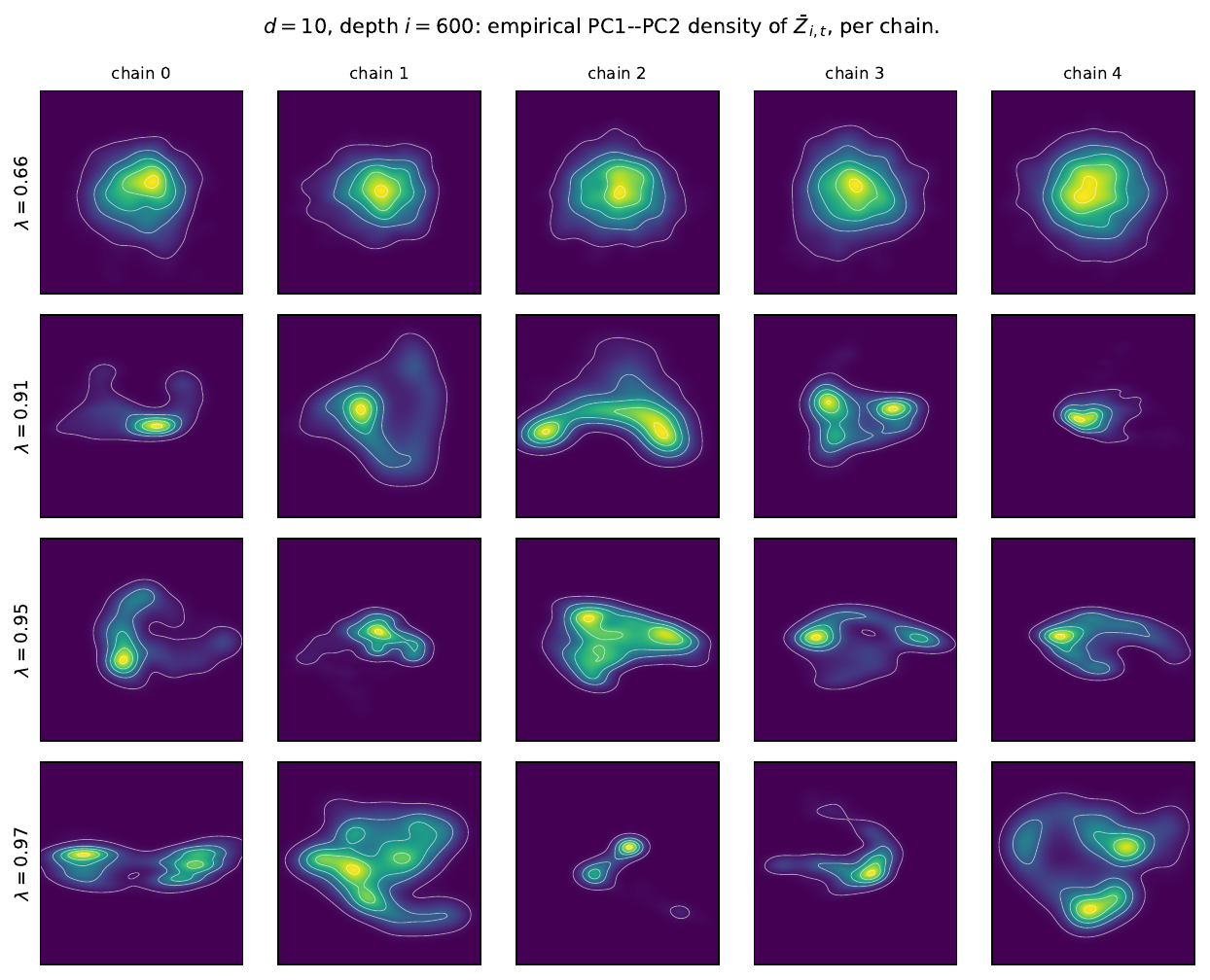}
  \caption{Empirical PC1--PC2 density (Gaussian KDE, Scott bandwidth) of the same per-chain projections shown in \Cref{fig:appendix_pi_Z_pca_d10}. The $\lambda = 0.66$ row is the near-independent Gaussian baseline; from $\lambda = 0.91$ upward the densities are clearly multimodal and chain-specific, and by $\lambda = 0.95, 0.97$ each chain shows two or three well-separated modes.}
  \label{fig:appendix_pca_density_d10}
\end{figure}

\begin{figure}[t]
  \centering
  \includegraphics[width=\textwidth]{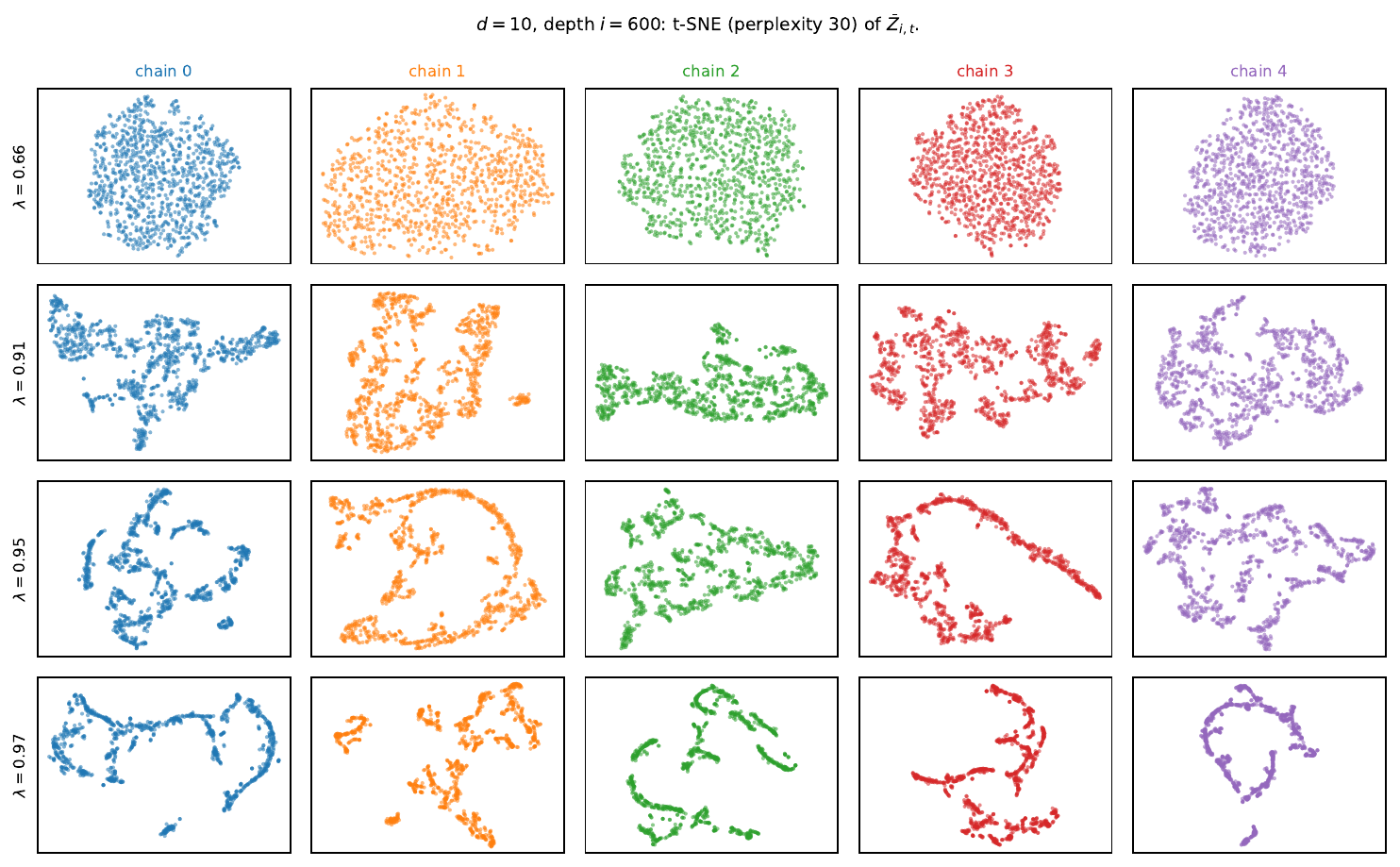}
  \caption{Per-chain t-SNE embeddings of $\bZ_{i,t}$ at $d = 10$, depth $i = 600$, perplexity $30$, PCA initialisation. Same $\lambda$ and chains as \Cref{fig:appendix_pi_Z_pca_d10}.}
  \label{fig:appendix_pi_Z_tsne_d10}
\end{figure}

\begin{figure}[t]
  \centering
  \includegraphics[width=\textwidth]{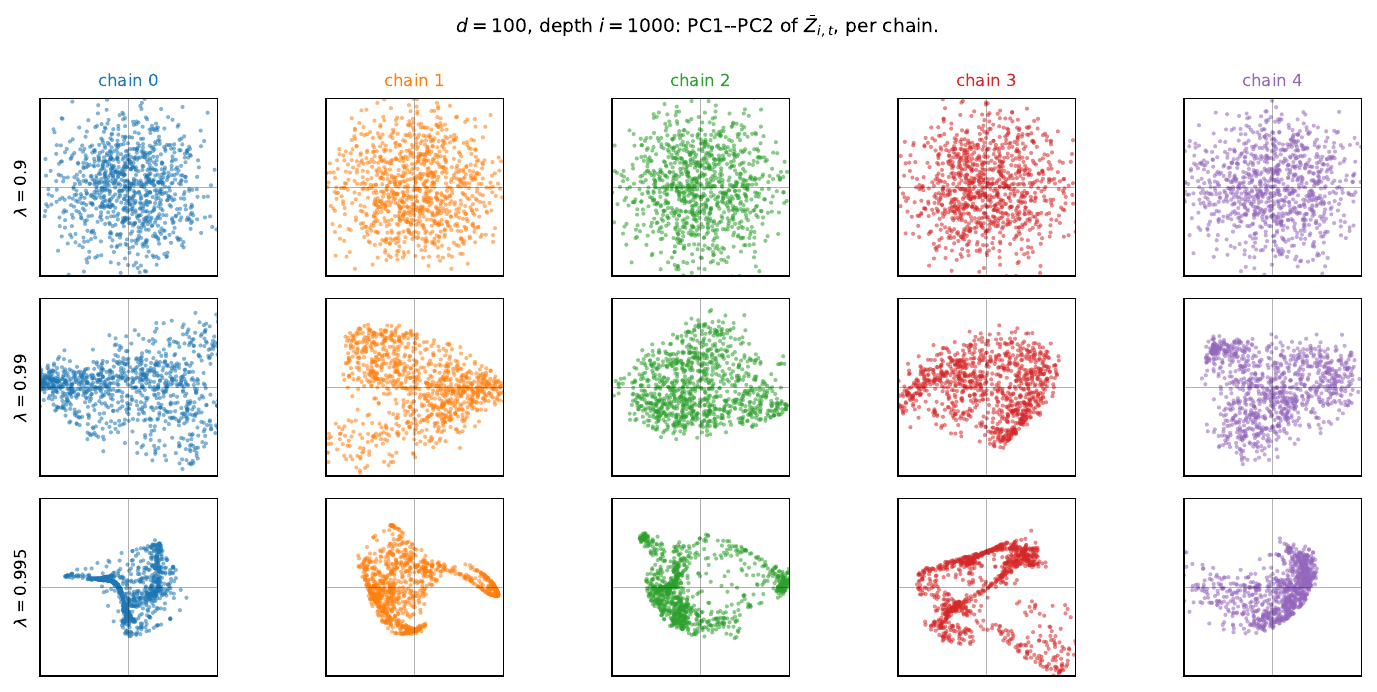}
  \caption{Per-chain PC1--PC2 scatters of $\bZ_{i,t} \in \RR^{100}$ at depth $i = 1000$, the linear-PCA companion to \Cref{fig:exp_pi_Z_tsne_main}. Same $\lambda$ values and same five i.i.d.\ chains as the main-paper t-SNE figure; $n = 1000$ points per panel, centred and projected onto the top two PCs. The chain-specific bananas, U-shapes, and arcs at $\lambda = 0.99, 0.995$ that appear in t-SNE are also visible here, so they reflect the data geometry rather than the embedding. See also \Cref{fig:appendix_pca_density_d100}, where we plot the \emph{density} of the points above, which makes the structure and multimodality even clearer.}
  \label{fig:appendix_pi_Z_pca_d100}
\end{figure}

\begin{figure}[t]
  \centering
  \includegraphics[width=\textwidth]{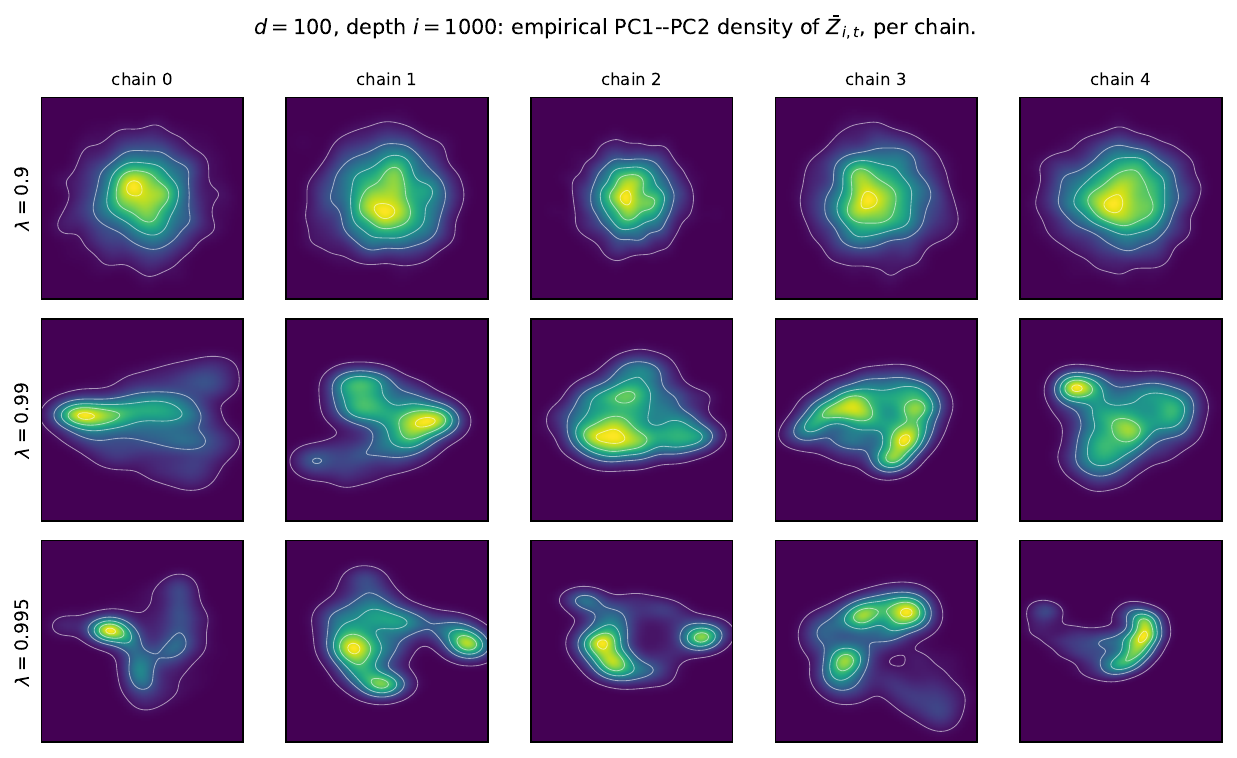}
  \caption{Empirical PC1--PC2 density (Gaussian KDE, Scott bandwidth) of the same per-chain projections shown in \Cref{fig:appendix_pi_Z_pca_d100}. The $\lambda = 0.90$ row is the near-independent Gaussian baseline; at $\lambda = 0.99$ and $\lambda = 0.995$ the densities show pronounced chain-specific multimodality with two or three well-separated peaks per chain.}
  \label{fig:appendix_pca_density_d100}
\end{figure}

\end{document}